\documentclass{article}
\input{preamble.tex}

\newtheorem{Theorem*}{Theorem}

\newtheorem{Claim*}[Theorem]{Claim}

\newtheorem{CounterExample*}{$\overline{\hbox{\bf Example}}$}

\newtheorem{Example*}[Theorem]{Example}

\newtheorem{Intuition*}[Theorem]{Intuition}
\newtheorem{Joke*}[Theorem]{Joke}

\newtheorem{Lemma*}[Theorem]{Lemma}
\newtheorem{Open problem}[Theorem]{Open problem}

\newtheorem{Question*}[Theorem]{Question}

\newcommand{\ignore}[1]{}

\newcommand{\cE}{{\mathcal E}}

\newcommand{\cO}{{\mathcal O}}

\newcommand{\cX}{{\mathcal X}}
\newcommand{\cY}{{\mathcal Y}}

\definecolor{light}{gray}{.75}

\def \upto  {{,}\ldots{,}}

\def \sets#1{{\{#1\}}}

\def \Paren#1{{\left({#1}\right)}}

\def \half    {{\frac12}}

\def\ignore#1{}

\newcommand{\bi}{\begin{itemize}}
\newcommand{\ei}{\end{itemize}}

\def\orpro{\mathop{\mathchoice
   {\vee\kern-.49em\raise.7ex\hbox{$\cdot$}\kern.4em}
   {\vee\kern-.45em\raise.63ex\hbox{$\cdot$}\kern.2em}
   {\vee\kern-.4em\raise.3ex\hbox{$\cdot$}\kern.1em}
   {\vee\kern-.35em\raise2.2ex\hbox{$\cdot$}\kern.1em}}\limits}

\def\andpro{\mathop{\mathchoice
 {\wedge\kern-.46em\lower.69ex\hbox{$\cdot$}\kern.3em}
 {\wedge\kern-.46em\lower.58ex\hbox{$\cdot$}\kern.25em}
 {\wedge\kern-.38em\lower.5ex\hbox{$\cdot$}\kern.1em}
 {\wedge\kern-.3em\lower.5ex\hbox{$\cdot$}\kern.1em}}\limits}

\def\simge{\mathrel{%
   \rlap{\raise 0.511ex \hbox{$>$}}{\lower 0.511ex \hbox{$\sim$}}}}

\def\simle{\mathrel{
   \rlap{\raise 0.511ex \hbox{$<$}}{\lower 0.511ex \hbox{$\sim$}}}}

\newcommand{\tcr}[1]{\textcolor{red}{#1}}

\usepackage{microtype}
\usepackage{graphicx}
\usepackage{subcaption}
\usepackage{booktabs} 

\usepackage{hyperref}
\hypersetup{colorlinks,linkcolor={blue},citecolor={blue},urlcolor={red}}  

\usepackage{comment}
\excludecomment{hide}
\includecomment{keep}

\usepackage{selectp}
\makeatletter
\def\namedlabel#1#2{\begingroup
   \def\@currentlabel{#2}%
   \label{#1}\endgroup
}
\makeatother

\setlist[enumerate]{leftmargin=5.5mm}
\title{Optimal Robust Learning of Discrete Distributions from Batches}

\begin{document}
\author{
Ayush Jain and Alon Orlitsky\\
  University of California, San Diego\\
  \texttt{\{ayjain,alon\}@eng.ucsd.edu}
}
\maketitle

\begin{abstract} 
Many applications, including natural language processing, sensor networks, collaborative filtering, and federated learning, call for estimating discrete distributions from data collected in batches, some  of which may be untrustworthy, erroneous, faulty, or even adversarial.

Previous estimators for this setting ran in exponential time, and for some regimes required a suboptimal number of batches. We provide the first polynomial-time estimator that is optimal in the number of batches and achieves essentially the best possible estimation accuracy.
\end{abstract}
\section{Introduction}
\subsection{Motivation}
Estimating discrete distributions from their samples is a fundamental modern-science tenet.
\cite{kamath2015learning} showed that as the number of sample $s$ grows, a $k$-symbol 
distribution can be learned to expected $L_1$ distance $\sim\sqrt{2(k\!-\!1)/(\pi s)}$ that 
we call the \emph{information-theoretic limit}.

In many applications, some samples are inadvertently or maliciously corrupted. A simple and intuitive example shows that this erroneous data limits the extent to which a distribution can be learned, even with infinitely many samples.

Consider the extremely simple case of just two possible binary distributions: $(1,0)$ and $(1-\advfrac,\advfrac)$. An adversary  who observes a $1-\beta$ fraction of the samples and can determine the rest, could use the observed samples to learn the underlying distribution, and set the remaining samples to make the distribution appear to be $(1-\advfrac,\advfrac)$. By the triangle inequality, even with arbitrarily many samples, any estimator for $\targetdis$ incurs an $L_1$ loss $\ge \advfrac$ for at least one of the two distributions. We call this the \emph{adversarial lower bound}.

The example may seem to suggest a pessimistic conclusion. If an adversary can corrupt a $\advfrac$ fraction of the data, a loss $\ge\advfrac$ is unavoidable. Fortunately, that is not necessarily so.

In many applications data is collected in batches, most of which are genuine, but some possibly corrupted. Here are a few examples.
Data may be gathered by sensors, each providing a large amount of data, and some sensors may be faulty. The word frequency of an author may be estimated from several large texts,  some of which are mis-attributed. Or user preferences may be learned by querying several users, but some users may intentionally 
bias their feedback.

Interestingly, even when a $\advfrac$-fraction of the batches are corrupted, the underlying distribution can be estimated to $L_1$ distance much lower than $\advfrac$.
Consider for example just three $n$-sample batches, of which one is chosen adversarially.
The underlying distribution 
can be learned from each genuine batch to expected $L_1$ distance $\sim\sqrt{2(k-1)/(\pi \bsize)}$.
It is easy to see that the average of the two estimates pairwise-closest in $L_1$ distance achieves a comparable expected distance that large batch size $n$ is much lower than $\advfrac$. 

This raises the natural question of whether estimates from even more batches can be combined effectively to estimate distributions to within a distance that is
not only much smaller than the $\advfrac$ achieved when no batch information was utilized, but also significantly smaller than the $O(\sqrt{k/n})$ distance derived above when two batches were used. For example can the underlying distribution be learned to a small $L_1$ distance when, as in many practical examples, $n\le k$?

To formalize the problem,~\cite{qiao2017learning}
considered learning a $\alphsize$-symbol distribution $\targetdis$ whose samples are provided in batches of size $\ge\bsize$. A total of $\btotal$ batches are provided, of which a fraction $\le\beta$ may be arbitrarily and adversarially corrupted, while in every other batch $b$ the samples are drawn according a distribution $\batchdis$ satisfying  $||\batchdis-\targetdis||_1\le\eta$, allowing for the possibility that slightly different distributions generate samples in each batch.

For this adversarial batch setting, they  showed that for any alphabet size $\alphsize\ge2$, and any number $\btotal$ of batches, the lowest achievable $L_1$ distance is $\ge \eta+ \frac\advfrac{\sqrt{2\bsize}}$. We refer to this as the \emph{adversarial batch lower bound}.

For $\advfrac\!<\!1/900$, they also derived an estimation
algorithm that approximates $\targetdis$ to $L_1$ distance
$O(\max\{\eta+\advfrac/\sqrt{\bsize},\allowbreak
\sqrt{(\bsize+\alphsize)/(\bsize\btotal)} 
\})$,
achieving the adversarial batch lower bound.
Surprisingly therefore, 
not only can the underlying distribution be approximated to $L_1$ distance $O(\sqrt{k/n})$ that falls below $\advfrac$, but the distance diminishes as $\advfrac/\sqrt n$, independent of the alphabet size $k$, given sufficient number of batches $m$.

Yet, the algorithm in~\cite{qiao2017learning} had three significant drawbacks. 1) it runs in time exponential in the alphabet size, hence  impractical for most relevant applications; 2) its guarantees are limited to very small fractions of corrupted batches $\advfrac\ge 1/900$, hence do not apply to practically important ranges; 3) with $\btotal$ batches of size $\ge\bsize$ each, the total number of samples is $\ge\bsize\btotal$, and for alphabet size $\alphsize \ll \bsize$, the  algorithm's distance guarantee falls short of the information-theoretic  $\Theta(\sqrt{\alphsize/(\bsize\btotal)})$ limit. 

In this paper we derive an algorithm that 1) runs in polynomial time in all parameters;
2) can tolerate any fraction of adversarial batches $\advfrac < 1/2$, though to derive concrete constant factors in the theoretical analysis, we assume $\advfrac \le 0.4$;
3) achieves distortion 
$O(\max\{\eta+ \advfrac\sqrt{\frac{\log(1/\advfrac)}{{n}}}, \sqrt{\frac{\alphsize}{\bsize\btotal}} \})$ that achieves the statistical limit in terms of the number $\bsize\btotal$ of samples, and is optimal up to a small $O(\sqrt{\log(1/\advfrac)})$ factor from the adversarial batch lower bound.

The algorithm's computational efficiency, enables the first experiments of learning with adversarial batches. We tested the algorithm on simulated data with various  adversarial-batch distributions  
and adversarial noise levels up to $\advfrac= 0.49$. The algorithm runs in a fraction of a second, and as shown in Section~\ref{sec:exp}, estimates $\targetdis$ nearly as well as an oracle that knows the identity of the adversarial batches.

To summarize, the algorithm runs in polynomial time, works for any adversarial fraction $\advfrac<0.5$, is optimal in number of samples, and essentially optimal in batch size. It opens the door to practical robust estimation in sensor networks, federated learning, and collaborative filtering. 

\ignore{
works for and  theoretical guarantees, and is highly practical in terms of all measures, namely run-time, number of batches required, and, the estimation error guarantee; hence attains the gold standard for the problem of robust distribution estimation in a Federated or distributed learning setting.}

\subsection{Problem Formulation} 

Let $\Delta_k$ be the collection of all distributions over $\alphabet=\sets{1\upto\alphsize}$. 
The $\Lone$ distance between two distributions $p,q\in\Delta_k$ is
\[
||p-q||_1
\triangleq
\sum_{i\in \alphabet} |p(i)-q(i)|
=
2\cdot\max_{ \alphsubsetdef} |p(\alphsubset)-q(\alphsubset)|.  
\]
We would like to estimate an unknown \emph{target distribution} $\targetdis\in\Delta_k$ to a small $\Lone$ distance from samples, some of which may be corrupted or even adversarial. 

Specifically, let $B$ be a collections of $m$ batches of $n$ samples each.
Among these batches is an unknown collection of \emph{good batches} $\bgood\subseteq B$; each batch $b\in \bgood$ in this collection has $n$ independent samples $X^{b}_1,X^{b}_2,...,X^{b}_n\sim \probestbatch$ with $||\probestbatch- \targetdis||_1\le \targetdistance $. Furthermore, the batches of samples in $\bgood$ are independent of each other. 

For the special case where $\targetdistance = 0$, all samples in the good batches are generated by the target distribution $\targetdis=\probestbatch$. 
Since the proofs and techniques are essentially the same for $\targetdistance = 0$ and $\targetdistance > 0$, for simplicity of presentation we assume that $\targetdistance = 0$. We briefly discuss, at the end, how these results translate to the case $\targetdistance > 0$.

The remaining set $\badv = \allbatches \setminus \bgood$ of \emph{adversarial batches} consists of arbitrary
$n$ samples each, that may even be chosen by an adversary, possibly based on the samples in the good batches.
Let $\goodfrac=\bgoodsize/\btotal$, 
and $\advfrac=\badvsize/\btotal=1-\goodfrac$ 
be the fractions of good and adversarial batches, respectively. 

Our goal is to use the $\btotal$ batches 
to return a distribution $\finaldis$ such that $||\finaldis-\targetdis||_1$ is small or equivalently $|\subsetprobtarget-\finaldis(\alphsubset)|$ is small for all $\alphsubsetdef$.

\begin{hide}
\subsection{Lower Bounds}
\cite{qiao2017learning} showed an adversarial lower bound on the $L_1$ distance achievable for the above problem. They also show an algorithm which achieves the lower bound within a constant factor although the run time of the algorithm is exponential in $\alphsize$. For completeness, we present the lower bound and the outline of \tcr{a somewhat simplified version of the }proof given in \cite{qiao2017learning}. Note that the lower bound is independent of the alphabet size $\alphsize$.
\begin{theorem} \cite{qiao2017learning} \label{th:lowbou}
For every $\advfrac \in [0,1/2), \,\alphsize\ge 2,\,\bsize\ge 1$, and $\btotal$, for any estimate $\hat p$ 
there is a distribution $p$ such that with probability $\ge 1/2$, $||\targetdis-\hat p||_1 \ge \frac\advfrac{\sqrt{2\bsize}}$.
\end{theorem}
\begin{proof}
Consider $\alphsize=2$. Let $ \gamma =\frac\advfrac{2\sqrt{2\bsize}}$, and let $\targetdis$ be either 
$\text{Bern}(\half+\gamma)$ or $\text{Bern}(\half-\gamma)$.

For every batch, the number of $1$'s is a sufficient statistic for estimating $\targetdis$, and it is distributed  either $B(\bsize, \half+\gamma)$ or $B(\bsize, \half-\gamma)$. The $L_1$ distance between  these distributions is small enough such that the adversary can choose distributions $q_1$ and $q_2$, over number of ones in the adversarial batches, such that
\[
(1-\advfrac)B(\bsize, \half+\gamma)+ \advfrac q_1=(1-\advfrac)B(\bsize, \half- \gamma)+\advfrac q_2. 
\]
We skip the simple proof of this statement. 
Hence, if the good batches are distributed as $B(\bsize, \half+\gamma)$ then adversary chooses $q_1$ as distribution of the adversarial batches and if good batches are distributed as $B(\bsize, \half-\gamma)$ then adversary chooses $q_2$ and in both the cases the resultant joint distribution of all the batches is same. Hence the two cases are indistinguishable. Then $||\text{Bern}(\half+\gamma)-\text{Bern}( \half-\gamma)||_1 = 2|\half+\gamma -  ( \half-\gamma) |= 4\gamma$. From the triangle inequality, any estimate $\hat \targetdis$ will be at a distance at least $\min\{||\text{Bern}(\half+\gamma)-\hat \targetdis||,\ ||\text{Bern}(\half-\gamma)-\hat \targetdis||_1\} \ge 2\gamma$ from one of the possible distribution of the batches.
\end{proof}
The theorem implies that even with access to infinitely many batches, even for an alphabet of size as small as $2$, no algorithm can estimate $\targetdis$ to  $L_1$ distance below 
$\Omega(\advfrac/\sqrt{\bsize})$ with probability $1/2$. In the next section, we show that this bound can essentially be met by a polynomial-time algorithm.

\end{hide}

\subsection{Result Summary}
In section~\ref{sec:algmajg} we derive a polynomial-time algorithm that returns an 
estimate $p^*$ of $\targetdis$ with the following properties.
\begin{theorem}\label{th:mainresult}
For any given $\advfrac\le 0.4$, $\bsize$, $\alphsize$, and
$\btotal\ge O(\frac{\alphsize }{\advfrac^2 \log(1/\advfrac)})$,
Algorithm~\ref{alg1} runs in time polynomial in all parameters and its estimate $p^*$ satisfies $||p^*-\targetdis||_1 \le 100 \advfrac \sqrt{\frac{\log (1/\advfrac)}{\bsize}}$ with probability $\ge 1- O(e^{-\alphsize})$.   
\end{theorem}
The  theorem implies that our algorithm can achieve the adversarial lower bound to a small factor of $O(\sqrt{\log(1/\advfrac)})$ using the optimal number of samples. The next corollary shows that when the number of samples is not enough to achieve the adversarial batch lower bound our algorithm achieves the statistical lower bound.

If the number of batches $\btotal \ll \frac{\alphsize }{\advfrac^2 \log(1/\advfrac)}$,  then let $\advfrac_*$ such that $\btotal=\Theta\Paren{ \frac{\alphsize}{\advfrac_*^2 \log(1/\advfrac_*)}}$.
Clearly, $\advfrac_*\gg \advfrac$.
From Theorem~\ref{th:mainresult}, the algorithm would achieve a distance $O\Paren{\advfrac_* \sqrt{\frac{\log (1/\advfrac_*)}{\bsize}}} = O\Paren{\sqrt{\frac{\alphsize}{\bsize\btotal}}}$.
We get the following corollary.
\begin{corollary}
For any given $\advfrac\le0.4$, $\bsize$ and $\alphsize$
and~$\btotal\gg\alphsize$, Algorithm~\ref{alg1} runs in polynomial time, and its estimate $p^*$ satisfies $||p^*-\targetdis||_1 \le O(\max\{\advfrac \sqrt{\frac{\ln (1/\advfrac)}{\bsize}}, \sqrt{\frac{\alphsize}{\btotal\bsize}}\})$ with probability $\ge 1-O(e^{-\alphsize})$.  
\end{corollary}

Note that our polynomial time algorithm achieves the statistical limits for $L_1$ distance and achieves the adversarial batch lower bounds to a small multiplicative factor of $O(\sqrt{\log(1/\advfrac)})$.  

\subsection{Comparison to Recent Results and techniques}

In a paper concurrent and independent of this work,~\cite{chen2019efficiently} propose an algorithm that uses the sum of squares methodology to estimates $\targetdis$ to the same distance as ours.
Their algorithm need $\tilde O( \frac{(\bsize\alphsize)^{O(\log(1/\advfrac))}}{\advfrac^4} )$ batches and has a run-time $\tilde O(\frac{(\bsize\alphsize)^{O(\log^2(1/\advfrac))}}{\advfrac^{O(\log (1/\advfrac))}} )$. Both the sample complexity and run time are much higher than ours, and is quasi-polynomial. They also consider certain structured distributions, not addressed in this paper, for which they provide an algorithm with similar run time, but lower sample complexity.

\subsection{Other Related Work}
The current results extend several long lines of work on learning distributions and their properties.

The best approximation of a distribution with a given number of samples was determined up to the exact first-order constant for KL loss~\cite{braess2004bernstein}, and $L_1$ loss and $\chi^2$ loss~\cite{kamath2015learning}. 
These settings do not allow adversarial examples, and some modification of the empirical estimates of the samples is often shown to be near optimal. 
This is not the case in the presence of adversarial samples, where the challenge is to devise algorithms that are efficient from both computational and sample viewpoints. 

Our results also relate to classical robust-statistics work~\cite{tukey1960survey,huber1992robust}. There has also been significant recent work leading to practical distribution learning algorithms 
that are robust to adversarial contamination of the data.
For example,~\cite{diakonikolas2016robust,lai2016agnostic} presented algorithms for learning the mean and covariance matrix of high-dimensional sub-gaussian and other
distributions with bounded fourth moments 
in presence of the adversarial samples. Their estimation guarantees are typically in terms of $L_2$, and do not yield the $L_1$- distance results required for discrete distributions. 

The work was extended in~\cite{charikar2017learning} to the case when more than half of the samples are adversarial.
Their algorithm returns a small set of candidate distributions one of which is a good approximate of the underlying distribution. For more extensive survey on robust learning algorithms in the continuous setting,  see~\cite{steinhardt2017resilience,diakonikolas2019robust}.

Another motivation for this work derives 
from the practical federated-learning problem, where information arrives in batches~\cite{mcmahan2016communication,mcmahan2017mac}. 
\subsection{Preliminaries}
We introduce notation that will help outline our approach and will be used in rest of the paper.

Throughout the paper, we use $\Bsc$ to denote a sub-collection of batches in $\allbatches$ and use $\Gsc$ and $\Asc$ for a sub-collection of batches in $\bgood$ and $\badv$, respectively. And $\alphsubset$ is used to denote a subset of $\alphabet$, we abbreviate singleton set of $\alphabet$ such as $\{j\}$ by $j$.

For any batch $b\in B$, we let $\Vbempprob$ denote the empirical measure defined by samples in batch $b$. And for any sub-collection of batches $\Bsc \subseteq \allbatches $, let $\VUempprob $ denote the empirical measure defined by combined samples in all the batches in $\Bsc$. We use two different symbols to distinguish the empirical distribution defined by an individual batch and the empirical distribution defined by a sub-collection of batches. Let $\mathbf{1}_S(.)$ denote the indicator random variable for set $\alphsubset$. Thus, for any subset $\alphsubsetdef$,
\[
\bempprob \triangleq \frac1\bsize\sum_{i\in [n]}\mathbf{1}_S(X_{i}^b) 
\]
and
\[
\Uempprob \triangleq  \frac1{|\Bsc|\bsize}\sum_{b\in \Bsc}\sum_{i\in [n]}\mathbf{1}_S(X_{i}^b) = \frac{1}{|\Bsc |}\sum_{b\in \Bsc }\bempprob.
\]
Note that $\VUempprob$ is the mean of the empirical measures $\Vbempprob$ defined by the batches $b\in\Bsc$.
For subset $\alphsubsetdef$, let $\med$ be the median of the set of estimates $\{\bempprob : b\in \allbatches \}$. Note that the median has been computed using the estimates $\bempprob$ for all the batches in $b\in \allbatches$.

For $r\in[0,1]$, we let $\var r \triangleq  \frac{r(1-r)}{\bsize}$,  which we use to denote the variance of sum of $\bsize$ i.i.d. random variables distributed according to $\text{Bernoulli}(r)$.

We pause briefly to note the following two properties of the function $\var r$ that we use later. 
\begin{equation}\label{eq:fineq}
 \forall\, r,s\in [0,1],\, \var r\le \frac{1}{4\bsize}\text{ and }|\var r-\var s|\le \frac{|r-s|}\bsize.
\end{equation}
Here the second property made use of the fact that the derivative $|V'(r)|\le 1/\bsize,\, \forall\, r \in [0,1] $.

For $b\in \bgood$, $\mathbf{1}_S(X_{i}^b)$ for $i\in[\bsize]$ are i.i.d. with distribution $\mathbf{1}_S(X_{i}^b) \sim \text{Bernoulli}(\subsetprobtarget)$. 
For $b\in \bgood$, since $\bempprob$ is average of $\mathbf{1}_S(X_{i}^b)$, $i\in[n]$, therefore, 
\[
E[\,\bempprob\,] = \subsetprobtarget \quad\text{ and } \quad E[(\bempprob-\subsetprobtarget)^2] = \var\subsetprobtarget.
\]

For any collection of batches $\Bsc  \subseteq \allbatches $ and subset $\alphsubsetdef$, the empirical probability $\bempprob$ of $\alphsubset$ based on batches $b\in \Bsc$ will differ for the different batches.
The empirical variance of these empirical probabilities $\bempprob$ for batches $b\in \Bsc$ is denoted as
\[
 \empvarsub {\Bsc} \triangleq  \frac{1}{|{\Bsc}|}\sum_{b\in {\Bsc}}(\bempprob- \Uempprob)^2. 
\]
\begin{hide}
\tcr{We use $\var{\Uempprob}$ as an estimate of the variance of $\bempprob$ for the batches in ${\Bsc}$ and will refer to it as the \emph{mean-induced variance estimate}. This is because in the absence of outliers, or ${\Bsc} = \bgood$, and $\bgoodsize$ large enough, $\empprobsub \bgood \rightarrow \subsetprobtarget$, hence, $\var{\empprobsub\bgood} \rightarrow \var\subsetprobtarget$.}
\end{hide}

\subsection{Organization of the Paper}
In Section~\ref{sec:algmajg} we present the algorithm, its analysis along with the key insights used in developing the algorithm. Section~\ref{sec:exp} reports the performance of the algorithm on experiments performed on the simulated data.

\begin{hide}
Section~\ref{sec:goodprop} 
shows that
with high probability, the collection of good batches satisfies some useful statistical properties.
Section~\ref{sec:algmajg} assumes that the good batches display these properties, and proposes an algorithm that can always estimate $\targetdis$ to a small $L_1$ distance for any choice of adversarial batches. 
This section rely on a key procedure that can efficiently find subsets $\alphabet$ whose statistical properties were disturbed by adversarial batches, we discuss this procedure in Section~\ref{sec:det}. 
In Section~\ref{sec:} we briefly discuss the modification required in the algorithm when the distribution of good batches may vary a little from the target distribution. 
\end{hide}

\begin{hide}
\section{Overview of the Technique}
We give a simplified and informal description of our algorithm to illustrate the key idea.
At a high level, our algorithm removes the adversarial batches which are "outliers", possibly loosing a small number of good batches as well in the process. This outlier removal method forms the backbone of many robust learning algorithms. Notably~\cite{?} used this idea to learn the mean of a high dimensional sub-gaussian distribution from to a small $L_2$ distance. 

The main challenge in designing a robust learning algorithm to find the outlier batches efficiently and a number of new ideas are needed to identify the outlier batches in the setting considered. 

We illustrate the difficulty of identifying the adversarial batches.
For an individual batch, in general, it is not possible to determine whether the batch consists of samples from the distribution or a distribution at a large TV distance. 
Even for the most good batch $\bempprob$ will differ significantly from $\subsetprobtarget$ and for some set $\alphsubset$ among the $2^\alphsize$ subsets of $\alphabet$, making it hard to identify the adversarial batches.
For example, consider the batches of samples from a uniform distribution over $\alphsize$ the empirical distribution of the samples in any batch of size $\bsize$ is at an $L_1$ distance $>2- 2\bsize/\alphsize$, which can be arbitrary close to 1, for the distributions with large domain size $\alphsize$. To address this challenge, we use the following observation.  

We observe that for a fixed subset $\alphsubsetdef$ and a good batch $b\in \bgood$, $\bempprob$ has a sub-gaussian distribution  $\text{subG}(\subsetprobtarget,\frac{1}{4{\bsize}})$ and the variance is $\var{\subsetprobtarget}$. Therefore, for a fixed subset $\alphsubset$ most of the good batches assign the empirical probability  $\bempprob \in \subsetprobtarget \pm O(1/\sqrt{\bsize})$. Moreover, the mean and the variance of $\bempprob$ for $b\in\bgood$ converges to the expected values $\subsetprobtarget$ and $\var{\subsetprobtarget}$, respectively. 

The collection of all batches $\allbatches$ also has $\advfrac-$fraction of unknown adversarial batches along with the good batches. If for the adversarial batches in $\allbatches$ the average difference between $\bempprob$ and $\subsetprobtarget$ is within a few standard deviations $\tilde O(\frac{1}{\sqrt{\bsize}})$, then these adversarial batches can only deviate the overall mean of empirical probabilities $\bempprob$ by $\tilde O(\frac{\advfrac}{\sqrt{\bsize}})$ from $\subsetprobtarget$. 
The combined mean of $\bempprob$ for all batches deviates significantly from $\subsetprobtarget$ only if a large number of adversarial batches $b\in\badv$ assigns empirical probability $\bempprob$ that differ from $\subsetprobtarget$ by quantity much larger than the standard deviation $\tilde O(\frac{1}{\sqrt{\bsize}})$.
We want to find such subsets $\alphsubset$ to identify the adversarial batches efficiently.
\end{hide}

\begin{hide}

To identify the affected subsets $\alphsubset$, we first note that the presence of such adversarial batches also increases the empirical variance $\empvarsub{\allbatches}$ of the empirical probabilities $\bempprob$ from its expected value $\var{\subsetprobtarget}$. 
But the affected subset $\alphsubset$ may still not have the largest empirical variance $\empvarsub{\allbatches}$ among the all subsets of $\alphabet$, because, the expected value of variance of $\bempprob$ for some subsets $\alphsubset$ is large while for other subsets it is small. Hence finding the subset $\alphsubset$ with the largest variance doesn't work. We address this challenge using the following key observation.

Recall that the mean of empirical probabilities $\bempprob$ for good batches $b$ converges, or equivalently $\empprobsub \bgood \rightarrow \subsetprobtarget$. This implies $\var{\empprobsub\bgood} \rightarrow \var\subsetprobtarget$. Also since the empirical variance $\empvarsub{\bgood}$ converges to $\var\subsetprobtarget$, we get $\empvarsub{\bgood}- \var{\empprobsub\bgood} \rightarrow 0$. Therefore, without corruption by the adversarial batches the difference between two estimators of the variance would be small for all subsets $\alphsubsetdef$, and its large value, we show in Lemma~\ref{lem?}, can reliably detect the adversarial corruption.

After reducing the problem of finding the affected subset to finding the subset with large value of the difference $\empvarsub{\bgood}- \var{\empprobsub\bgood}$, we propose an efficient way to find such a subset efficiently among $2^k$ subsets in Section~\ref{sec:det}, which may be of independent interest.

Once the algorithm finds the affected subset $\alphsubset$, it identifies a sub-collection of outlier batches that has adversarial batches in a larger proportion.
These outlier batches roughly correspond to the batches for which the absolute difference between $\bempprob$ and $\subsetprobtarget$ is much larger than the typical value $\tilde O(\frac{1}{\sqrt{\bsize}})$. 
Algorithm remove these outlier batches and repeat this procedure on the remaining batches to remove outliers for other subsets $\alphsubset$.

After removing outlier for each affected subset $\alphsubset$, we finally get $\Bsc_f$ such that $|{\bar{\targetdis}_{\Bsc_f}(\alphsubset)} -\subsetprobtarget| \le \tilde O(\advfrac\sqrt{\frac{\log (1/\advfrac)}\bsize})$ for all $\alphsubsetdef$. In the next section, we give the pseudocode of the algorithm along with the key lemmas used in analyzing the algorithm.

Looking at this difference allows us to find more subtle changes efficiently, which were not detectable due to the variance varying across different $\alphsubset$. So if algorithm looks only for subsets whose variance is unusually high than adversary can affect the em

Since without the adversarial corruption this quantity is uniformly small for all subsets it allows us to detect more subtle deviations.

\end{hide}

\section{Algorithm and its Analysis}\label{sec:algmajg}

At a high level, our algorithm removes the adversarial batches --- which are "outliers" --- possibly losing a small number of good batches as well in the process. The outlier removal method forms the backbone of many robust learning algorithms. Notably~\cite{diakonikolas2016robust,diakonikolas2017being} have used this idea to learn the mean of a high dimensional sub-gaussian distribution up to a small $L_2$ distance, even in an adversarial setting.
The main challenge in designing a robust learning algorithm is actually the task of finding the outlier batches efficiently. Several new ideas are needed to identify the outlier batches in the setting considered here. 

\begin{hide}
If the alphabet size is small then probably one can achieve this successfully.   
We illustrate this with a simple examples, consider $\alphsize =2$ and $\targetdis = (p(1),p(2))$. To estimate $\targetdis$ is same as estimating $p(1)$. To estimate $p(1)$ first note that the sufficient statistics for a batch $b$ is number of ones divided by the number of samples $n$ in batch, which is $\Vbempprob(1)$. Observe that for $b\in\bgood$, $n\Vbempprob(1)\sim \text{Ber}(p(1),n)$. $\Vbempprob(1)$ has sub-gaussian distribution with $\subG(p(1),1/4n)$. If the adversarial batches also have $\Vbempprob(1)$, close to 
\end{hide}

We begin by illustrating the difficulty of identifying the adversarial batches. Even if $\targetdis$ is known, in general, one cannot determine whether a batch $b$ has samples from $\targetdis$ or from a distribution at a large $L_1$ distance from $\targetdis$. The key difficulty is that,
for a batch having $\bsize$ samples from $\targetdis$, typically the difference between $\bempprob$ and $\subsetprobtarget$ is large for some of the subsets among $2^\alphsize$ subsets of $\alphabet$.
For example, consider batches of samples from a uniform distribution over $\alphsize$.  The empirical distribution of the samples in any batch of size $\bsize$ is at an $L_1$ distance $\ge2(1- \bsize/\alphsize)$, which for the distributions with large domain size $\alphsize$ can be up to two, which is the maximum $L_1$ distance between two distributions. To address this challenge, we use the following observation.  

For a fixed subset $\alphsubsetdef$ and a good batch $b\in \bgood$, $\bempprob$ has a sub-gaussian distribution  $\text{subG}(\subsetprobtarget,\frac{1}{4{\bsize}})$ and the variance is $\var{\subsetprobtarget}$. Therefore, for a fixed subset $\alphsubset$, most of the good batches assign the empirical probability  $\bempprob \in \subsetprobtarget \pm \tilde O(1/\sqrt{\bsize})$. Moreover, the mean and the variance of $\bempprob$ for $b\in\bgood$ converges to the expected values $\subsetprobtarget$ and $\var{\subsetprobtarget}$, respectively. 

The collection of batches $\allbatches$ along with good batches also includes a sub-collection $\badv$ of adversarial batches that constitute up to an $\advfrac-$fraction of $\allbatches$. If for adversarial batches $b\in\badv$, the average difference between $\bempprob$ and $\subsetprobtarget$ is within a few standard deviations $\tilde O(\frac{1}{\sqrt{\bsize}})$, then these adversarial batches can only deviate the overall mean of empirical probabilities $\bempprob$ by $\tilde O(\frac{\advfrac}{\sqrt{\bsize}})$ from $\subsetprobtarget$. 
Hence, the mean of $\bempprob$ will deviates significantly from $\subsetprobtarget$ only if for a large number of adversarial batches $b\in\badv$ empirical probability $\bempprob$ differ from $\subsetprobtarget$ by quantity much larger than the standard deviation $\tilde O(\frac{1}{\sqrt{\bsize}})$.

We quantify this effect by defining the \emph{corruption score}. 
For a subset $\alphsubsetdef$, let 
\[
\med \triangleq \text{median}\{\bempprob : b\in \allbatches\}.
\]
For a subset $\alphsubsetdef$ and a batch $b$, \emph{corruption score} $\indcorruption$ is defined as
\[
\indcorruption
\triangleq
\begin{cases}
0, \ \ \ \text{ if }\ |\bempprob- \med| \le 3 \sqrt{\frac{\ln (6e/\advfrac)}\bsize} ,\\
(\bempprob- \med)^2, \ \ \  \text{ else}.
\end{cases}
\]
Because $\subsetprobtarget$ is not known, the above definition use median of $\bempprob$ as its proxy. 

From the preceding discussion, it follows that for a fixed subset $\alphsubsetdef$, corruption score of most good batches w.r.t. $\alphsubset$ is zero, 
and adversarial batches that may have a significant effect on the overall mean of empirical probabilities have high corruption score $\indcorruption$.

\begin{hide}
For a subset $\alphsubsetdef$, let 
\[
U(\alphsubset) \triangleq \{b\in \allbatches : |\bempprob- \med| \ge 4 \sqrt{\frac{\ln (6e/\advfrac)}\bsize} \}.
\]
$U(\alphsubset)$ is a collection of \emph{suspicious batches} for a set $\alphsubset$. 
\end{hide}

The \emph{corruption score} of a sub-collection $\Bsc$ w.r.t. a subset $\alphsubset$ is defined as the sum of the \emph{corruption score} of batches in it, namely
\[
\genoverallcorruptionsub  \triangleq \sum_{b\in \Bsc} \indcorruption.
\]
A high corruption score of $\Bsc$ w.r.t. a subset $\alphsubset$ indicates the presence of many batches $b\in\Bsc$ for which the difference $|\bempprob- \med|$ is large. Finally, for a sub-collection $\Bsc$ we define \emph{corruption} as 
\[
\genoverallcorruption \triangleq\max_{\alphsubsetdef}\genoverallcorruptionsub .
\]
Note that removing batches from a sub-collection reduces corruption. We can simply make corruption zero by removing all batches, but we would lose all the information as well. 
The proposed algorithm reduces the corruption below a threshold by removing a few batches while not sacrificing too many good batches in the process.  

The remainder of this section assumes that the sub-collection of good batches $\bgood$ satisfies certain deterministic conditions.  
Lemma~\ref{lem:prophold} shows that the stated conditions hold with high probability for sub-collection of good batches in $\bgood$. 
Nothing is assumed about the adversarial batches, except that they form a $\le\advfrac$ fraction of the overall batches $\allbatches$.

\textbf{Conditions:} Consider a collection of $\btotal$ batches $\allbatches$, each containing $\bsize$ samples. Among these batches, there is a collection $\bgood \subseteq \allbatches$ of good batches of size $\bgoodsize\ge (1-\advfrac)\btotal$ and a distribution $\targetdis \in \Delta_\alphsize$ such that the following deterministic conditions hold for all subsets $\alphsubsetdef$:
\begin{enumerate}
    \item\label{con} The median of the estimates $\{\bempprob: b\in \allbatches\}$ is not too far from $\subsetprobtarget$.
    \[|\med-\subsetprobtarget|\le \sqrt{\ln (6)/\bsize}.
    \]
    \item\label{con2} For all sub-collections $\Gsc\subseteq\bgood$ of good batches of size $|\Gsc| \ge (1-\advfrac/6)\bgoodsize$, 
\begin{align*}
&|\UGempprob - \subsetprobtarget | \le  \frac \advfrac 2\sqrt{\frac{\ln (6e/\advfrac) } {\bsize}},  \\
  &\Big|\frac1{|\Gsc |}  \sum_{b\in \Gsc } (\bempprob -   \subsetprobtarget)^2  - \var{\subsetprobtarget}\Big|\le {\frac{ 6\advfrac\ln (\frac{6e}\advfrac) } {\bsize}}.
\end{align*}
    \item\label{con3} The corruption for good batches $\bgood$ is small, namely
\begin{align*}
\corruption(\bgood) \le {\frac{ \advfrac \btotal\ln ({6e}/\advfrac) } {\bsize}}.
\end{align*}
\end{enumerate}

Condition~\ref{con} and~\ref{con3} above are self-explanatory. Condition~\ref{con2} illustrates that for any sub-collection of good batches that retains all but a small fraction of good batches, empirical mean and variance estimate the actual values $\subsetprobtarget$ and $\var{\subsetprobtarget}$.

\begin{lemma}\label{lem:prophold}
When samples in $\bgood$ come from $\targetdis$ and $\bgoodsize \ge O(\frac{\alphsize}{\advfrac^2\ln (1/\advfrac)})$, then conditions~\ref{con}-~\ref{con3} hold simultaneously with probability $\ge 1-O(e^{-\alphsize})$.
\end{lemma}
We prove the above lemma by using the observation that for $b\in\bgood$, $\bempprob$ has a sub-gaussian distribution $\text{subG}(\subsetprobtarget, \frac{1}{4\bsize})$, and it has variance $\var{\subsetprobtarget}$.
The proof is in Appendix~\ref{sec:goodprop}.


For easy reference, in the remaining paper, we will denote the upper bound in Condition~\ref{con3} on the corruption of $\bgood$ as
\[
\kappa_G\triangleq{\frac{ \advfrac \btotal\ln ({6e}/\advfrac) } {\bsize}}.
\]
Assuming that the above stated conditions hold, the next lemma bounds the $L_1$ distance between the empirical distribution $\VUempprob$ and $\targetdis$ for any sub-collection $\Bsc$ in terms of how large its corruption is compared to $\kappa_G$.  
\begin{lemma}\label{lem:corruptionandell1}
Suppose the conditions~\ref{con}-~\ref{con3} holds. Then for any $\Bsc$ such that $|\Bsc\cap\bgood| \ge (1 -\frac\advfrac6) \bgoodsize$ and let $\genoverallcorruption =t\cdot\kappa_G$, for some $t\ge 0$,
then
\[
||\VUempprob - \targetdis ||_1 \le  (10+3\sqrt{t}) \advfrac\sqrt{{\frac{\ln (6e/\advfrac)}\bsize} }.
\]
\end{lemma}
Observe that for any sub-collection $\Bsc$ retaining a major portion of good batches, from condition~\ref{con2}, the mean of $\Vbempprob$ of the good batches $\Bsc\cap\bgood$  approximates $\targetdis$.
Then showing that a small corruption score of $\Bsc$ w.r.t. all subsets $\alphsubset$ imply that the adversarial batches $\Bsc\cap\badv$ have limited effect on $\VUempprob$ proves the above lemma. A complete proof is in Appendix~\ref{App:proofs}. 

We next exhibit a Batch Deletion procedure in~\ref{alg2} that lowers the corruption score of a sub-collection $\Bsc$ w.r.t. a given subset $\alphsubset$ by deleting a few batches from the sub-Collection. This will be a subroutine of our main algorithm.
Lemma~\ref{lem:alg2char} characterizes its performance.
\begin{algorithm}[H]
  \caption{Batch Deletion}\label{alg2}
  \begin{algorithmic}[1]
    \STATE {\bfseries Input:} Sub-Collection of Batches $\Bsc$, subset $\alphsubsetdef$, $\medV$=$\med$, and $\advfrac$.
  \STATE {\bfseries Output:} A collection $DEL\subseteq \Bsc$ of batches to delete.
    \STATE ${DEL} = \{\}$;
    \WHILE{  $\genoverallcorruptionsub \, \geq\,  20\kappa_G$} 
        \STATE {Samples batch $b\in \Bsc$ such that probability of picking a batch $b\in \Bsc$ is $ \frac{\indcorruption}{\genoverallcorruptionsub }$};
        \STATE $DEL \gets DEL \cup {b}$\,;
        \STATE $\Bsc \gets \{\Bsc \setminus {b}\}$;
    \ENDWHILE
    \STATE \textbf{return }$(DEL)$;
  \end{algorithmic}
\end{algorithm}

\begin{lemma}\label{lem:alg2char}
For a given $\Bsc$ and subset $\alphsubset$ procedure~\ref{alg2} returns a sub-collection $DEL\subset \Bsc$, such that
\begin{enumerate}
    \item For subset $\alphsubset$ the corruption score $\corruption(\Bsc\setminus DEL,\alphsubset)$ of the new sub-collection is $<20\kappa_G$.
    \item Each batch $b\in \Bsc$ that gets included in $DEL$ is an adversarial batch with probability $\ge 0.95$.
    \item The subroutine deletes at-least $\genoverallcorruptionsub - 20\kappa_G$ batches. 
\end{enumerate}
\end{lemma}
\begin{proof}
Step 4 in the algorithm ensures the first property. Next, to prove property 2, we bound 
the probability of deleting a good batch as
\[
\sum_{b\in \Bsc\cap\bgood}\frac{\indcorruption}{\genoverallcorruptionsub } \le \frac{\sum_{b\in \bgood}\indcorruption}{\genoverallcorruptionsub }
\le \frac{\kappa_G }{20\kappa_G}, 
\]
here the last step follows from condition~\ref{con3} and while loop conditional in step 4.
Property 3 follows from the observation that the total corruption score reduced is $\ge (\genoverallcorruptionsub - 20\kappa_G)$ and corruption score of one batch is bounded as $\indcorruption\le 1$.
\end{proof}

We will use procedure~\ref{alg2} to successively update $\allbatches$ to decrease the corruption score for different subsets $\alphsubsetdef$. The next lemma show that even after successive updates the resultant sub-collection retains most of the good batches.
\begin{lemma}\label{lem:bounddel}
Let $\Bsc$ be the sub-collection after applying
any number of successive deletion updates 
suggested by the Algorithm~\ref{alg2}
on $\allbatches$, for any sequence of input subsets $\alphsubset_1,\alphsubset_2,....\subseteq \alphabet$, then $|\Bsc\cap\bgood| \ge (1 -\advfrac/6) \bgoodsize$, with probability $\ge 1-O(e^{-\alphsize})$. 
\end{lemma}

Therefore, one can make successive updates to the collection of all batches $\allbatches$ by deleting the batches suggested by procedure~\ref{alg2} for all subsets in $\alphsubsetdef$ one by one. This will result in a sub-collection $\Bsc\subseteq \allbatches$, which still has most of the good batches and corruption score $\genoverallcorruptionsub $ bounded w.r.t. each subset $\alphsubset$. However, this will take time exponential in $\alphsize$ as there are $2^\alphsize$ subsets, and therefore, we want a computationally efficient method to find a subsets $\alphsubset$ with high corruption score and use procedure~\ref{alg2} for only those subsets.
Next, we derive a novel method to achieve this objective.

We start with the following observation. A high corruption score of sub-collection $\Bsc$ with respect to an affected subset $\alphsubset$ implies a higher empirical variance of $\bempprob$ for such $\alphsubset$ than the expected value of the variance of $\bempprob$. While an affected subset $\alphsubset$ the empirical variance $\empvarsub{\allbatches}$ is higher than expected, it is not necessarily higher than the empirical variance observed for all non-affected subset.
This is because $\var\subsetprobtarget$, the expected value of the variance of $\bempprob$, for some subsets $\alphsubset$ may be larger compared to the other. Hence, simply finding the subset $\alphsubset$ with the largest variance doesn't work.

We use the following key insight to address this. Recall that the mean of empirical probabilities $\bempprob$ for good batches $b\in\bgood$ converges, or equivalently $\empprobsub \bgood \rightarrow \subsetprobtarget$. This implies that $\var{\empprobsub\bgood} \rightarrow \var\subsetprobtarget$. Also, since the empirical variance $\empvarsub{\bgood}$ converges to $\var\subsetprobtarget$, we get $\empvarsub{\bgood}- \var{\empprobsub\bgood} \rightarrow 0$. Therefore, without corruption by the adversarial batches the difference between two estimators of the variance would be small for all subsets $\alphsubsetdef$, and its large value, we show in Lemma~\ref{lem:trans}, can reliably detect any significant adversarial corruption. This happens because empirical variance of $\bempprob$ depends on the second moment whereas the other estimator $\var{\Uempprob}$ of variance depends on the mean of $\bempprob$, hence the corruption affects the second estimator less severely. The next Lemma shows that the difference between the two variance estimators for subset $\alphsubset$ can indicate the corruption score w.r.t. subset $\alphsubset$

\begin{lemma}\label{lem:trans}
Suppose the conditions~\ref{con}-~\ref{con3} holds. Then for any $\Bsc\subseteq \allbatches$ such that $|\Bsc\cap\bgood| \ge (1 -\frac\advfrac6) \bgoodsize$ and let $\genoverallcorruptionsub =t\cdot\kappa_G$ for some $t\ge 0$, then following holds.
\[
\empvarsub {\Bsc}- \var{\Uempprob} \le \Big(t+4\sqrt{t}+ 28\Big) \kappa_G,
\]
\[
\empvarsub {\Bsc}- \var{\Uempprob}\ge \Big(0.5 t-8\sqrt{t}-25\Big)\kappa_G.
\]
\end{lemma}

The next Lemma shows that a subset for which $\empvarsub {\Bsc}- \var{\Uempprob}$ is large, can be found using a polynomial-time algorithm. In subsection~\ref{sec:det} we derive the algorithm. We refer to this algorithm as~\emph{$Detection-Algorithm$}. The next lemma characterizes the performance of this algorithm. In subsection~\ref{sec:det}, we show that the algorithm achieves the performance guarantees of the next Lemma.
\begin{lemma}~\label{lem:polalg}
$Detection-Algorithm$ has run time polynomial in number of batches in its input sub-collection $\Bsc$ and alphabet size $\alphsize$, and returns $\alphsubset^*_{{\Bsc}}$ such that
\begin{align*}
|\overline{{\text{V}}}_{{\Bsc}}(\alphsubset^*_{{\Bsc}})- \var{\VUempprob(\alphsubset^*_{{\Bsc}})}| &
\\
\ge 0.56\max_{\alphsubsetdef}&| \empvarsub {\Bsc}- \var{\Uempprob}|.   
\end{align*}
\end{lemma}

This leads us to the Robust distribution Learning Algorithm~\ref{alg1}. Theorem~\ref{th:algper} characterizes its performance. 
\begin{algorithm}
\caption{Robust Distribution Estimator}\label{alg1}
  \begin{algorithmic}[1]
    \STATE {\bfseries Input:} All batches $b\in \allbatches $, batch size $\bsize$, alphabet size $\alphsize$, and $\advfrac$.
  \STATE {\bfseries Output:} Estimate $p^*$ of the distribution $\targetdis$.
    \STATE $i\gets 1$ and ${\Bsc_i} \gets \allbatches $.
    \WHILE{ True }
        \STATE $\alphsubset^*_{{\Bsc_i}}=Detection-Algorithm(\Bsc_i)$
        \IF{$|\Delta_{\Bsc_i}(\alphsubset^*_{{\Bsc_i}})| \le 75\kappa_G$ } 
        \STATE Break;
        \ENDIF
        \STATE $\medV \gets \medV(\bar{\mu}(\alphsubset^*_{{\Bsc_i}}))$.\\ 
        \STATE $DEL\gets $\text{Batch-Deletion}{($\Bsc_i,\alphsubset^*_{{\Bsc_i}},\medV$)}.
    \ENDWHILE\\
    \STATE \textbf{return }$(\targetdis^* \gets \empprob {\Bsc_i})$.
  \end{algorithmic}
\end{algorithm}

\begin{theorem}\label{th:algper}
Suppose the conditions~\ref{con}-~\ref{con3} holds. Then Algorithm~\ref{alg1} runs in polynomial time and return a sub-collection $\Bsc_{f} \subseteq\allbatches$ such that $|\Bsc_f\cap\bgood| \ge (1 -\frac\advfrac6) \bgoodsize$ and for $\targetdis^* = {\bar{\targetdis}_{\Bsc_f}}$,
\[
||\targetdis^* - \targetdis ||_1 \le  100 \advfrac \sqrt{\frac{\ln (6e/\advfrac) } {\bsize}}.
\]
\end{theorem}
\textbf{Outline of the Proof of Theorem~\ref{th:algper}:} 
In each round of the algorithm, 
Subroutine $Detection-Algorithm$ 
finds subsets for which the difference between the two variance estimates is large. Lemma~\ref{lem:trans} implies that the corruption w.r.t. this subset is large. The deletion subroutine updates the sub-collection of batches by removing some batches from it and reduces the corruption w.r.t. the detected subset $\alphsubset$.

The algorithm terminates when for some sub-collection $\Bsc_f$ subroutine $Detection-Algorithm$ returns a subset $\alphsubset$ small difference between the two variance estimators. Then Lemma~\ref{lem:polalg}  implies that the difference is small for all subsets. Lemma~\ref{lem:trans} further implies that if the difference between the two variance estimators is small then the corruption is bounded w.r.t. all subsets for sub-collection $\Bsc_f$. Finally, Lemma~\ref{lem:corruptionandell1} bounds the $L_1$ distance between $\empprob {\Bsc_f}$ and $\targetdis$.$\hfill\square$

Combining Lemma~\ref{lem:prophold} and Theorem~\ref{th:algper} yields Theorem~\ref{th:mainresult}.

\subsection{Extension to $\targetdistance> 0$}\label{subsec:etanonzero}
Recall that when $\eta>0$, for each good batch $b\in\bgood$, the distribution $\batchdis$ of samples in batch $b$ is close to the common target distribution $\targetdis$, such that $||\batchdis-\targetdis||\le \eta$, instead of necessarily being the same. For simplicity, we have given the algorithm and the proof for only $\eta = 0$. The algorithm and the proof naturally extend to this more general case; here we get an extra additive dependence on $\eta$ for the bounds in the lemmas and the theorems, and for the parameters of the algorithm. And with this slight modification in the parameters algorithm estimates $\targetdis$ to a distance $O(\eta+\advfrac\sqrt{\ln(1/\advfrac)/{\bsize}})$, and has the same sample and time complexity.

\subsection{Efficient Detection Algorithm}\label{sec:det}
In this subsection, we derive the procedure $Detection-Algorithm$, that runs in the polynomial time and achieves the performance in Lemma~\ref{lem:polalg}.

Given a collection ${\Bsc}$ of batches, we construct two covariance matrices $C^{EV}_{\Bsc}$ and $C^{EM}_{\Bsc}$ of size $\alphsize\times\alphsize$. 

For an alphabet size $\alphsize$, we can treat the empirical probabilities estimates $\Vbempprob$ and $\VUempprob$ as a $\alphsize$-dimensional vector such that $j^{th}$ entry denote the empirical probability of the $j^{th}$ symbol. 
Recall that $\VUempprob$ is the mean of $\Vbempprob$, $b\in \Bsc$.

The first covariance matrix, $C^{EV}_{\Bsc}$, is the covariance matrix of $\Vbempprob$ for $b\in \Bsc$, with entries for $j,l\in\alphabet$,
\[
C^{EV}_{\Bsc}(j,l) = \frac 1{|{\Bsc}|}\sum_{b\in {\Bsc}} (\Vbempprob(j)-\VUempprob(j))(\Vbempprob(l)-\VUempprob(l)).
\]

The second covariance matrix $C^{EM}_{\Bsc}$, is an expected covariance matrix of $\Vbempprob$ if samples in the batches $b$ were drawn from the distribution $\VUempprob$. Hence, its entries are 
\[
C^{EM}_{\Bsc}(j,l) = -\frac{\VUempprob(j)\VUempprob(l)}{\bsize} \text{ for }j,l\in\alphabet,\, j\neq l, 
\]
and
\[
C^{EM}_{\Bsc}(j,j) = \frac{\VUempprob(j)(1-\VUempprob(j))}{\bsize}.
\]
Let $D_{\Bsc}$ be the difference of the two matrices:
\[
D_{\Bsc} = C^{EV}_{\Bsc}-C^{EM}_{\Bsc}.
\]

For a vector $x\in \{0,1\}^\alphsize$, let 
\[
S(x) \triangleq \{j\in\alphabet: x(j)= 1 \},
\]
be the subset of $\alphabet$ corresponding to the vector $x$.

\paragraph{Observations}
\begin{enumerate}
\item
The sum of elements in any row and or column for both the covariance matrices, and hence also for the difference matrix, is zero, hence 
\[
C^{EV}_{\Bsc} \boldsymbol{1}= C^{EM}_{\Bsc} \boldsymbol{1} = D_{\Bsc} \boldsymbol{1} = \boldsymbol0.
\]
\emph{Proof:}
We show for $C^{EV}_{\Bsc}$, the proof for $C^{EM}_{\Bsc}$ is similar. For any $j\in\alphabet$,
\begin{align*}
  &\sum_{l\in \alphabet} C^{EV}_{\Bsc}(j,l) \\
  &= \frac 1{|{\Bsc}|}\sum_{l\in \alphabet}\sum_{b\in {\Bsc}}
  (\Vbempprob(j)-\VUempprob(j))(\Vbempprob(l)-\VUempprob(l))
  \\
    &= \sum_{b\in {\Bsc}}(\Vbempprob(j)-\VUempprob(j))\sum_{l\in \alphabet} (\Vbempprob(l)-\VUempprob(l))\\
    &= \sum_{b\in {\Bsc}}(\Vbempprob(j)-\VUempprob(j)) (1-1)=0.
    \end{align*}
\item
\label{ob:3}
It is easy to verify that for any vector $x\in \{0,1\}^\alphsize$,
\begin{align*}
    \langle C^{EV}_{\Bsc},xx^\intercal\rangle & = \frac{1}{|{\Bsc}|}\sum_{b\in {\Bsc}} (  \Vbempprob(S(x))- \VUempprob(S(x)))^2\\
&= \overline{\text{V}}_{{\Bsc}}(\alphsubset(x)),
\end{align*}
the empirical variance of $\Vbempprob(S(x))$ for $b\in {\Bsc}$.
Similarly,
\begin{align*}
\langle C^{EM}_{\Bsc},xx^\intercal\rangle &= \frac{\VUempprob(S(x))(1-\VUempprob(S(x)))}{\bsize} \\
&= \var{\VUempprob(S(x))}.
\end{align*}

Therefore,
\begin{align*}
    \langle D_{\Bsc},xx^\intercal\rangle&=\langle C^{EV}_{\Bsc}-C^{EM}_{\Bsc},xx^\intercal\rangle\\
    &= \overline{\text{V}}_{{\Bsc}}(\alphsubset(x))-\var{\VUempprob(S(x))}.
\end{align*}
\item
Note that $y\rightarrow \frac{1}{2}(y+\boldsymbol{1}) $ is a 
1-1 mapping from $\{-1,1\}^\alphsize \rightarrow \{0,1\}^\alphsize$,
and that
\begin{align*}
    &\langle C^{EV}_{\Bsc},\frac{1}{2}(y+\boldsymbol{1})\frac{1}{2}(y+\boldsymbol{1})^\intercal\rangle
    \\
    &= \langle C^{EV}_{\Bsc},\frac{1}{4}(yy^\intercal +\boldsymbol{1}y^\intercal+y\boldsymbol{1}^\intercal+\boldsymbol{1}\boldsymbol{1}^\intercal)\rangle \\
    &=  \frac{1}{4}\langle C^{EV}_{\Bsc},yy^\intercal\rangle. 
\end{align*}
\end{enumerate}
Let 
\[
y = \arg\max_{y\in \{-1,1\}^\alphsize} |\langle D_{\Bsc},yy^\intercal\rangle|. 
\]
Then from $y$ one can recover the corresponding subset $S(x)$, with $x=\frac{1}{2}(y+\boldsymbol{1})$,
maximizing
\[
|\overline{\text{V}}_{{\Bsc}}(\alphsubset(x))-\var{\VUempprob(S(x))}|.
\]
In~\cite{alon2004approximating}, Alon et al. derives a polynomial-time approximation algorithm for the above optimization problem. The algorithm first uses a semi-definite relaxation of the problem and then uses randomized integer rounding techniques based on Grothendieck's Inequality.
Their algorithm recovers $y_{\Bsc}$ such that 
\[
|\langle D_{\Bsc},y_{\Bsc} y_{\Bsc}^\intercal\rangle| \ge 0.56 \max_{y\in\{-1,1\}^\alphsize} |\langle D_{\Bsc},yy^\intercal\rangle|.
\]
Let $x_{\Bsc} = \frac{1}{2}(y+\boldsymbol{1})$. Then from observation 3 it follows that
\[
|\langle D_{\Bsc},x_{\Bsc} x_{\Bsc}^\intercal\rangle| \ge 0.56 \max_{x\in\{0,1\}^\alphsize} |\langle D_{\Bsc},xx^\intercal\rangle|.
\]
Therefore for $\alphsubset^*_{\Bsc} = \alphsubset(x_{\Bsc})$ we get 
\begin{align*}
&|\overline{\text{V}}_{{\Bsc}}(\alphsubset^*_{\Bsc})-\var{\VUempprob(\alphsubset^*_{\Bsc})}|\\
&\ge 0.56\max_{\alphsubsetdef} |\overline{\text{V}}_{{\Bsc}}(\alphsubset)-\var{\VUempprob(\alphsubset)}|.
\end{align*}

\section{Experiments}\label{sec:exp}
We evaluate the algorithm's performance on synthetic data.

We compare the estimator's 
performance with two others: 
1) an oracle that knows the identity of the adversarial batches. The oracle ignores the adversarial batches and computes the empirical estimators based on remaining batches and is not affected by the presence of adversarial batches. The estimation error achieved by the oracle is the best one could get, even without the adversarial corruptions. 
2) a naive-empirical estimator that computes the empirical distribution of all samples across all batches.

Two non-trivial estimators have been derived for this problem. Both have prohibitively large sample and/or computational complexity.
The estimator in~\cite{qiao2017learning} has run time exponential in $\alphsize$, making it impractical. The time and sample complexities of the estimator in~\cite{chen2019efficiently} are either super-polynomial or a high-degree polynomial, depending on the range of the parameters  ($\alphsize$,$\bsize$,$1/\advfrac$), rendering their simulation prohibitively high as well. 

We tried different adversarial distributions and
found that the major determining factor of the effectiveness of the adversarial batches is the distance between the
adversarial distribution and the target distribution. 
If the adversarial distribution is too far, then adversarial batches are easier to detect. For this scenario our algorithm is even more effective than the performance limits shown in Theorem~\ref{th:mainresult} and the performance between our algorithm and the oracle is almost indistinguishable.
When the adversarial distribution is very close to the target distribution $\targetdis$, the adversarial batches don't affect the estimation error by much. The estimator has the worst performance when the adversary chooses the distribution of its batches at an optimal distance from target distribution. This optimal distance differs with the value of the algorithm's  parameters. Hence for each choice of algorithm parameters, we tried adversarial distributions at varying distances and reported the worst performance of our estimator.

\begin{figure*}
\begin{subfigure}{.45\textwidth}
  \centering
  \includegraphics[width=1.1\linewidth]{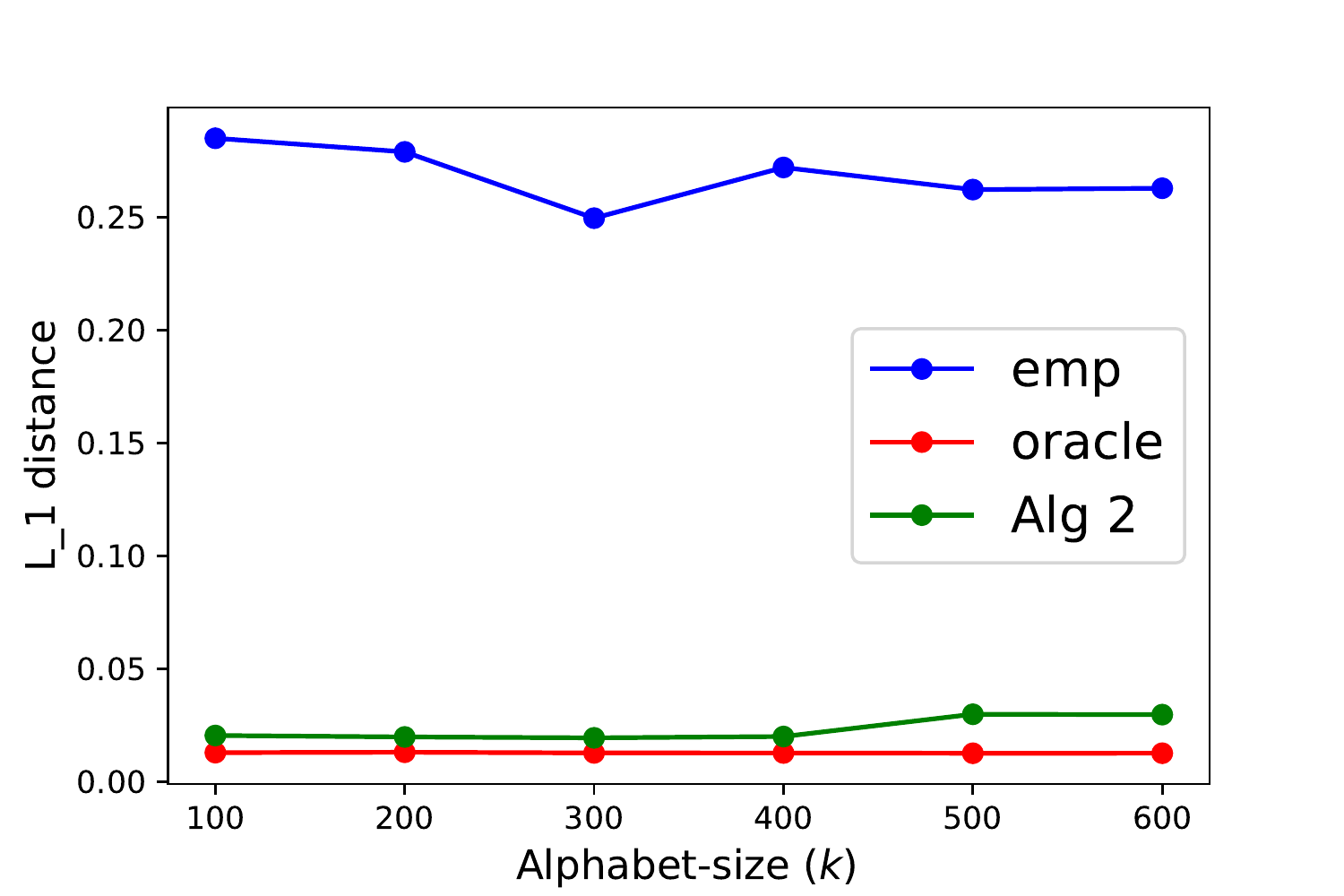}
  \caption{Support size $\alphsize$}
  \label{fig:sfig2}
\end{subfigure}
\begin{subfigure}{.45\textwidth}
  \centering
  \includegraphics[width=1.1\linewidth]{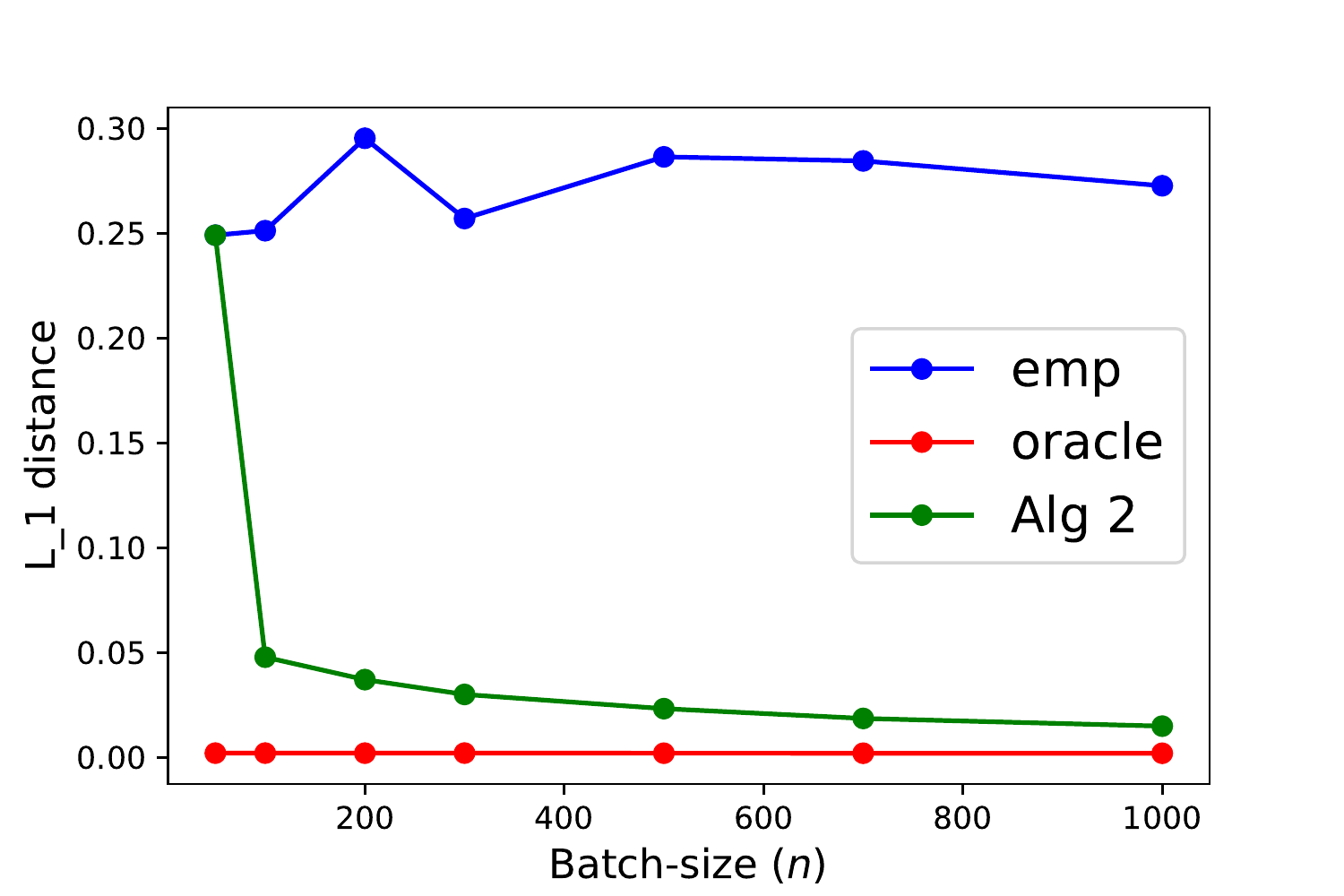}
  \caption{Batch size $\bsize$}
  \label{fig:sfig4}
\end{subfigure}
\begin{subfigure}{.45\textwidth}
  \centering
  \includegraphics[width=1.1\linewidth]{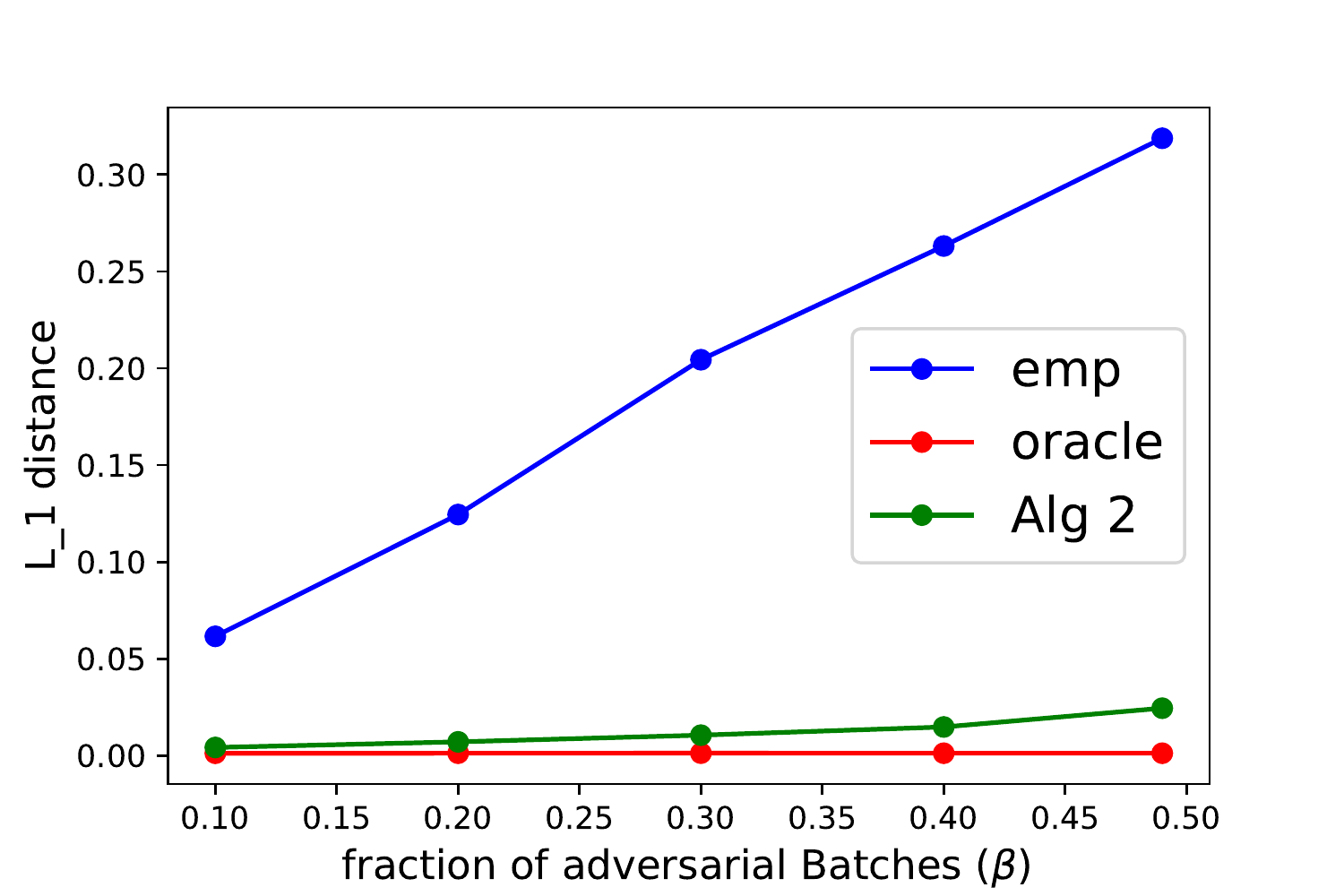}
  \caption{Adversarial batches fraction $\advfrac$}
  \label{fig:sfig1}
\end{subfigure}
\hskip 0.5in
\begin{subfigure}{.45\textwidth}
  \centering
  \includegraphics[width=1.1\linewidth]{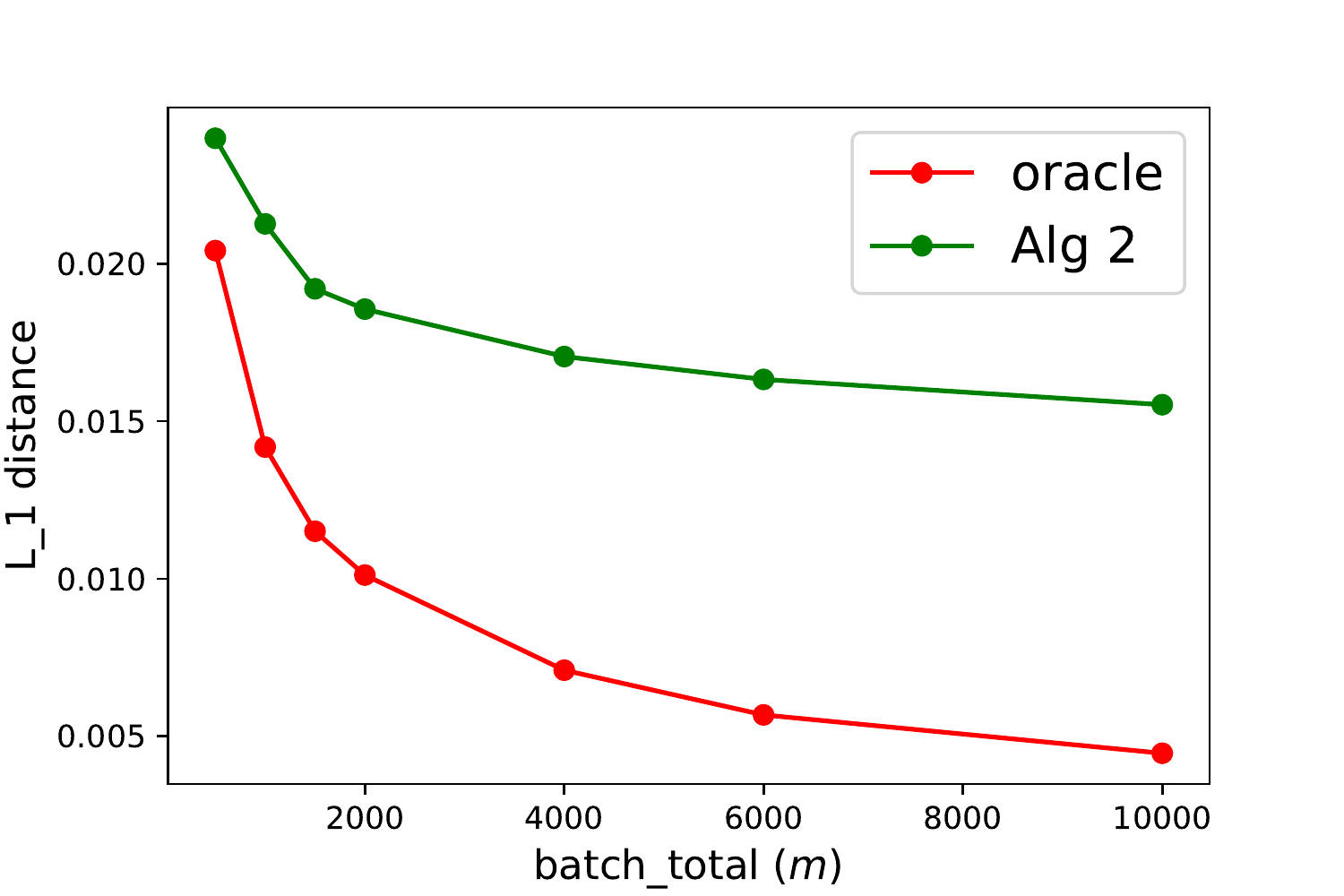}
  \caption{Number of batches $\btotal$}
  \label{fig:sfig3}
\end{subfigure}
\vskip 0.2in
\caption{$L_1$ estimation error with different Parameters}
\label{fig:fig}
\vskip 0.0in
\end{figure*}

All experiments were performed on a laptop with a configuration of 2.3 GHz Intel Core i7
CPU and 16 GB of RAM. We choose the parameters for the algorithm by using a small simulation. We provide all codes and implementation details in the supplementary material.

We show four plots here. In each plot we vary one parameter and plot the $L_1$ loss incurred by all three estimators. For each experiment, we ran ten trials and reported the average $L_1$ distance achieved by each estimator. 

For the first plot we fix batch-size $\bsize = 1000$ and $\advfrac = 0.4$ and vary alphabet size $\alphsize$. We generate $\btotal= \alphsize/(0.4)^2$ batches for each $\alphsize$. Our algorithm's performance show no significant change as the size of alphabet increases and its performance nearly matches the performance of the Oracle and outperforms the naive estimator by order of magnitudes.  

In the the second plot we fix $\advfrac=0.4$ and $\alphsize=200$ and vary batch size $\bsize$. We choose $\btotal = 40 \times \frac\alphsize{\beta^2}\times\frac{1000}\bsize$, this keeps the total number of samples $\bsize\times\btotal$, constant for different $\bsize$. We see that the $L_1$ loss incurred by our estimator is much smaller than the naive empirical estimator and it diminishes as the batch size increases and comes very close to the performance of the oracle. Note that this roughly matches the decay $O(1/\sqrt{\bsize})$ of $L_1$ error characterized in both the lower and the upper bounds.  

For the next plot we fix batch size $\bsize = 1000$ and $\alphsize=200$. The number of good batches $(1-\advfrac)\btotal= 400\alphsize$ is kept same. We vary the adversarial noise level and plot the performance of all estimators. We tested our estimator for fraction of adversarial batches as high as $0.49$ and still our estimator recovered $\targetdis$ to a good accuracy and in fact at the lower noise level it is essentially similar to  the oracle and it increases (near) linearly with the noise level $\advfrac$ as in Theorem~\ref{th:mainresult},

In the last plot we fixed all other parameters $\bsize = 1000$, $\alphsize=200$, and $\advfrac = 0.4$ and varied the number of batches. We see that the performance of oracle keep improving as number of bathes increases. But for our algorithm it decreases initially but later it saturates as predicted by adversarial batch lower bound.


\section*{Acknowledgements}
We thank Vaishakh Ravindrakumar and Yi Hao for helpful comments in the prepration of this manuscript.

We are grateful to the National Science
Foundation (NSF) for supporting this work through grants
CIF-1564355 and CIF-1619448.

\bibliographystyle{alpha}
\bibliography{ref}
\newpage

\onecolumn
\appendix
\section{Proof of Lemma ~\ref{lem:prophold}}\label{sec:goodprop}
In this section, we show that conditions~\ref{con}-\ref{con3} holds with high probability and prove Lemma~\ref{lem:prophold}.
To prove the lemma we first prove three auxiliary lemmas; each of these three Lemma will lead to one of the three conditions in Lemma~\ref{lem:prophold}. These three lemmas characterizes the statistical properties of the collection of good batches $\bgood$. We state and prove these lemmas in the next subsection.
\subsection{Statistical Properties of the Good Batches}

Recall that, for a good batch $b\in \bgood$ and subset $\alphsubsetdef$, $\mathbf{1}_S(X_{i}^b)$, for $i\in [\bsize]$, are i.i.d. indicator random variables and $\bempprob$ is the mean of these $\bsize$ indicator variables. Since the indicator random variables are sub-gaussian, namely $\mathbf{1}_S(X_{i}^b)\sim \text{subG}(\subsetprobtarget,1/4)$, the mean $\bempprob$ satisfies
$\bempprob \sim \text{subG}(\subsetprobtarget ,1/4\bsize)$. $ \text{subG}(.)$ is used to denote a sub-gaussian distribution.
This observation plays the key role in the proof of all three auxiliary lemmas in this section.

The first lemma among these three lemma show that for any fixed subset $\alphsubsetdef$, $\bempprob$ for most of the good batches is close to $\subsetprobtarget$. This lemma is used to show Condition~\ref{con}.
\begin{lemma}\label{lem:withinmean}
For any $\epsilon\in (0,1/4]$ and $\bgoodsize \ge 12\alphsize/\epsilon$, $\forall\,\alphsubsetdef $, 
with probability $\ge 1-e^{-\alphsize}$,
\begin{align}
    \big|\big\{b\in \bgood: |\bempprob-\subsetprobtarget|\ge \sqrt{\frac{\ln (1/\epsilon)}\bsize} \big\}\big|\le \epsilon\bgoodsize. \nonumber
\end{align}
\end{lemma}
\begin{proof}
From Hoeffding's inequality, for $b\in \bgood$ and $\alphsubsetdef$,  
\[
\Pr\Big[ |\bempprob-\subsetprobtarget| \ge  {\sqrt{\frac{\ln (1/\epsilon)}\bsize}} \Big] \le 2 e^{- 2 \ln (1 /\epsilon) } \le 2{\epsilon}^{2} \le \epsilon/2 . 
\]
Let $\mathbbm{1}_b(\alphsubset)$ be the indicator random variable that takes the value 1 iff $|\bempprob-\subsetprobtarget| \ge  \sqrt{\ln (1/\epsilon)/\bsize}$. Therefore, for $b\in \bgood$, $E[\mathbbm{1}_b(\alphsubset) ] \le \epsilon/2 $.
Using the Chernoff bound,
\[
\Pr[ \sum_{b\in\bgood}\mathbbm{1}_b(\alphsubset) \ge 
{\epsilon} \bgoodsize ]
\le  e^{- \frac 1 3 \cdot\frac {\epsilon} 2 \bgoodsize}
\le  e^{-  2\alphsize} .
\]
Taking the union bound over all $2^\alphsize$ subsets $\alphsubset$ completes the proof.
\end{proof}

The next lemma show that even upon removal of any small fraction of good batches from $\bgood$, the empirical mean and the variance of the remaining sub-collection of batches approximate the distribution mean and the variance well enough.

\begin{lemma}\label{lem:22concentration}
For any $\epsilon\in (0,1/4]$, and $\bgoodsize \ge \frac{\alphsize}{\epsilon^2\ln (e/\epsilon)}$. Then $\forall\, \alphsubsetdef$ and $\forall \,\Gsc \subseteq \bgood$ of size $|\Gsc |\ge (1- \epsilon)\bgoodsize$, 
with probability $\ge 1-6e^{-\alphsize}$,
\begin{align}\label{eq:lem1main}
   \Big|\UGempprob- \subsetprobtarget | \le 3\epsilon\sqrt{\frac{\ln (e/\epsilon) } {\bsize}}
\end{align}
and
\begin{align}\label{eq:lem1main2}
   \Big|\frac{1}{|\Gsc |}\sum_{b\in \Gsc } (\bempprob- \subsetprobtarget)^2 - \var{\subsetprobtarget}\Big|\le 32{\frac{ \epsilon\ln (e/\epsilon) } {\bsize}}.
\end{align}
\end{lemma}

\begin{proof}
From Hoeffding's inequality,
\begin{align}\label{eq:lem1all}
\Pr\Big[\bgoodsize|\bar p_{\bgood} (S) - \subsetprobtarget | \ge \bgoodsize {\epsilon}&\sqrt{\frac{\ln (e/\epsilon) } {\bsize}} \Big] = \Pr\Big[\big|\sum_{b\in \bgood} (\bempprob - \subsetprobtarget)\big| \ge \bgoodsize {\epsilon}\sqrt{\frac{\ln (e/\epsilon) } {\bsize}} \Big]\nonumber\\
&\le 2 e^{-\frac{\bgoodsize\epsilon^2}{2/(4\bsize)}\cdot{\frac{\ln (e/\epsilon) } {\bsize}}} = 2 e^{-{2\bgoodsize\epsilon^2}\ln (e/\epsilon)} \le 2 e^{-2\alphsize}.
\end{align}
Similarly, for a fix sub-collection $U_{G} \subseteq \bgood$ of size $1\le |U_{G}| \le  \epsilon\bgoodsize$,
\begin{align*}
\Pr\Big[|U_{G}|\cdot|\bar p_{U_{G}} (S) - \subsetprobtarget|\ge {\epsilon\bgoodsize}\sqrt{\frac{ \ln (e/\epsilon) } {\bsize}}\Big] &= 
\Pr\Big[\Big|\sum_{b\in U_{G}} (\bempprob- \subsetprobtarget)\Big| \ge {\epsilon\bgoodsize}\sqrt{\frac{ \ln (e/\epsilon) } {\bsize}} \Big] \\
&\le 2 e^{-2\ln(e/\epsilon) \frac{(\epsilon\bgoodsize)^2}{|U_{G}|}}
\le 2 e^{- 2\epsilon\bgoodsize\ln(e/\epsilon)},    
\end{align*}
where the last inequality used $|U_{G}|\le\epsilon\bgoodsize$. Next, the number of sub-collections (non-empty) of $\bgood$ with size $\le \epsilon\bgoodsize$ is bounded by
\begin{align}\label{eq:numberofsub}
\sum_{j=1}^{\lfloor \epsilon\bgoodsize\rfloor}  {\bgoodsize \choose j} 
\le  \epsilon\bgoodsize {\bgoodsize \choose {\lfloor \epsilon\bgoodsize\rfloor}}
\le \epsilon\bgoodsize \Big(\frac{e \bgoodsize }{\epsilon\bgoodsize}\Big)^{\epsilon\bgoodsize} \le  e^{\epsilon\bgoodsize\ln (e/\epsilon)+\ln(\epsilon\bgoodsize) }<  e^{\frac 3 2\epsilon\bgoodsize\ln (e/\epsilon)},    
\end{align}
where last of the above inequality used ${\ln (\epsilon\bgoodsize)} < {\epsilon\bgoodsize}/2 $ and $\ln(e/\epsilon)\ge 1$.
Then, using the union bound, $\forall\ U_{G} \subseteq \bgood$ such that $|U_{G}| \le \epsilon\bgoodsize$, we get
\begin{align}\label{eq:lem1sub}
\Pr\Big[|U_{G}|\cdot|\bar p_{U_{G}} (S) - \subsetprobtarget |\ge {\epsilon\bgoodsize}\sqrt{\frac{ \ln (e/\epsilon) } {\bsize}}\Big]
\le  2 e^{-\frac 1 2\epsilon\bgoodsize\ln (e/\epsilon)}< 2 e^{-\frac\alphsize{2\epsilon}}  < 2 e^{-2 \alphsize}.
\end{align}
For any sub-collection $\Gsc \subseteq \bgood$ with $|\Gsc |\ge (1-\epsilon)\bgoodsize$, 
\begin{align*}
    |\sum_{b\in \Gsc } (\bempprob - \subsetprobtarget)| &= |\sum_{b\in \bgood} (\bempprob - \subsetprobtarget)-\sum_{b\in \bgood/\Gsc } (\bempprob - \subsetprobtarget)|\\
    &\le \Big|\sum_{b\in \bgood} (\bempprob - \subsetprobtarget)\Big|+\Big|\sum_{b\in \bgood/\Gsc } (\bempprob - \subsetprobtarget)\Big|\\
    &\le |{\bgood}| \times |\bar p_{\bgood} (S) - \subsetprobtarget|+\max_{U_{G}: |U_{G}|\le \epsilon\bgoodsize } |U_{G}| \times|\bar p_{U_{G}} (S) - \subsetprobtarget|\\
    &\le {2\epsilon\bgoodsize}\sqrt{\frac{\ln (e/\epsilon)} \bsize},
\end{align*}
with probability $\ge 1-2 e^{-2\alphsize}-2 e^{-2\alphsize}\ge 1-4 e^{-2\alphsize}$. Then
\begin{align*}
    |\UGempprob - \subsetprobtarget| &= \frac{ 1}{|\Gsc |}\Big|\sum_{b\in \Gsc } (\bempprob - \subsetprobtarget)\Big| \le 2\frac{\epsilon\bgoodsize}{|\Gsc |}\sqrt{\frac{\ln (e/\epsilon)} \bsize}\\
    &\le
    \frac{2\epsilon}{(1-\epsilon)}\sqrt{\frac{\ln (e/\epsilon) }\bsize }
    <
    {3\epsilon}\sqrt{\frac{\ln (e/\epsilon) }\bsize },
\end{align*}
with probability $\ge 1-4 e^{-2\alphsize}$. 
The last step used $\epsilon \le 1/4$. 
Since there are $2^{\alphsize}$ different choices for $\alphsubsetdef$, from the union bound we get,
\[
\Pr\Big[\bigcup_{\alphsubsetdef} \Big \{|\UGempprob- \subsetprobtarget| > {4\epsilon}\sqrt{\frac{\ln (e/\epsilon) }\bsize } \Big \}\Big]\le 4 e^{-2\alphsize}\times 2^{\alphsize}= 4 e^{-\alphsize}.
\]
This completes the proof of \eqref{eq:lem1main}. 

Let $Y_b = (\bempprob- \subsetprobtarget)^2 - \var{\subsetprobtarget} $.
For $b\in \bgood$,  $\bempprob- \subsetprobtarget \sim \text{subG}(1/4\bsize)$, therefore 
\[(\bempprob- \subsetprobtarget)^2 -E (\bempprob- \subsetprobtarget)^2 = Y_b
\sim \text{subE}(\frac {16}{4\bsize}) = \text{subE}(\frac {4}{\bsize}).\] 
Here $\text{subE}$ is sub exponential distribution \cite{rigoll2015stat}. Then Bernstein's inequality gives:
\[
\Pr[\Big|\sum_{b\in \bgood} Y_b\Big| \ge 8 \bgoodsize \frac {\epsilon} {{\bsize}}\ln (e/\epsilon)  ] \le 2 e^{-\frac{\bgoodsize}2\big(\frac{8\epsilon\ln (e/\epsilon)/n}{4/n}\big)^2}=2 e^{-2{\bgoodsize}\epsilon^2\ln^2 (e/\epsilon)} \le 2 e^{-2\alphsize}.
\]
Next, for a fix sub-collection $U_{G} \subseteq \bgood$ of size $1\le |U_{G}| \le  \epsilon\bgoodsize$,
\begin{align*}
\Pr\Big[\Big|\sum_{b\in U_{G}} Y_b\Big| \ge  16{\epsilon\bgoodsize}{\frac{ \ln (e/\epsilon) } {\bsize}} \Big] 
&\le 2 e^{-\frac{16{\epsilon\bgoodsize}{\frac{ \ln (e/\epsilon) } {\bsize}}}{2\times 4/n}}\\
&\le 2 e^{-2\epsilon\bgoodsize\ln(e/\epsilon)}.    
\end{align*}
Then following the same steps as in the proof of \eqref{eq:lem1main} one can complete the proof of \eqref{eq:lem1main2}.
\end{proof}

To state the next lemma, we make use of the following definition.
For a subset $\alphsubsetdef$, let \[
\bgood^{d}(\alphsubset,\epsilon)\triangleq \big\{b\in \bgood: |\bempprob-\subsetprobtarget|\ge 2 \sqrt{\frac{\ln (6e/\epsilon)}\bsize}) \big\}
\]
be the sub-collection of batches for which empirical probabilities $\bempprob$ are far from $\subsetprobtarget$ for a given set $\alphsubset$. 

The last lemma of the section upper bounds the total squared deviation of empirical probabilities $\bempprob$ from $\subsetprobtarget$ for batches in sub-collection $\bgood^{d}(\alphsubset,\epsilon)$. It helps in upper bounding the corruption for good batches and show that Condition~\ref{con3} holds with high probability.

\begin{lemma}\label{lem:lemsqdev}
For any $0<\epsilon< 1/2$, and $\bgoodsize \ge \frac{120\alphsize}{\epsilon\ln (e/\epsilon)}$. Then $\forall\, \alphsubsetdef$, 
with probability $\ge 1-2e^{-\alphsize}$,
\begin{align}
    |\bgood^{d}(\alphsubset,\epsilon)|\le \frac{\epsilon}{40}\bgoodsize,
\end{align}
and
\begin{align}\label{eq:lemsqudev}
  \sum_{b\in \bgood^{d}(\alphsubset,\epsilon)} (\bempprob- \subsetprobtarget)^2 < \frac{\epsilon}{2}\bgoodsize {\frac{ \ln (e/\epsilon) } {\bsize}}.
\end{align}
\end{lemma}
\begin{proof}
The proof of the first part is the same as (with different constants) Lemma~\ref{lem:withinmean} and we skip it to avoid repetition.

To prove the second part we bound the total squared deviation of any subset of size $\le \frac{\epsilon}{40}\bgoodsize$.

Let $Y_b = (\bempprob- \subsetprobtarget)^2 - \var{\subsetprobtarget} $.
Similar to the previous lemma, for a fix sub-collection $U_{G} \subseteq \bgood$ of size $1\le |U_{G}| \le  \frac{\epsilon}{40}\bgoodsize$, Bernstein's inequality gives:
\begin{align*}
\Pr\Big[\Big|\sum_{b\in U_{G}} Y_b\Big| \ge  8{\frac\epsilon{20}\bgoodsize}{\frac{ \ln (e/\epsilon) } {\bsize}} \Big] 
&\le 2 e^{-\frac{8{\epsilon\bgoodsize}{\frac{ \ln (e/\epsilon) } {\bsize}}}{20\times 2\times 4/n}}\\
&\le 2 e^{-\frac\epsilon{20}\bgoodsize\ln(e/\epsilon)}.    
\end{align*}
From~\eqref{eq:numberofsub}, there are $ e^{\frac 3 {80}\epsilon\bgoodsize\ln (e/\epsilon)}$ many sub-collections of size $\le  \frac{\epsilon}{40}\bgoodsize$. 
Then taking the union bound for all sub-collections of this size  and all subsets $\alphsubsetdef$ we get,
\[
\Big|\sum_{b\in U_{G}} \Big((\bempprob- \subsetprobtarget)^2 - \var{\subsetprobtarget}\Big)\Big|\le {\frac{2\epsilon}{5}\bgoodsize}{\frac{ \ln (e/\epsilon) } {\bsize}},
\]
for all $U_G$ of size $\le  \frac{\epsilon}{40}\bgoodsize$. Then using the fact that $\var .$ is upper bounded by $\frac1{4\bsize}$, and therefore $|U_G|\var{\subsetprobtarget}\le \frac{\epsilon}{4\times40}\bgoodsize$, completes the proof.
\end{proof}

\subsection{Completing the proof of Lemma~\ref{lem:prophold}}

We first show condition~\ref{con} holds with high probability.

It is easy to verify that $|\subsetprobtarget- \med|\ge \sqrt{\ln 6/\bsize }  $, only if the sub-collection $T = \{b: |\subsetprobtarget- \bempprob|\ge \sqrt{\ln 6/\bsize }\}$ has at-least $0.5\btotal$ batches. But
\[
    |T| = |T\cap\bgood|+|T\cap\badv| \overset{\text{(a)}}< |\bgood|/6+|\badv| = \frac\btotal 6+\frac 5 6 |\badv| \overset{\text{(b)}}\le  \frac\btotal 6+ \frac{2 \btotal}{6} = 0.5\btotal ,
\]
where inequality (a) follows from Lemma~\ref{lem:withinmean} by choosing $\epsilon = 1/6$ and (b) follows since $\badvsize\le \advfrac\btotal\le 0.4\btotal$.

Using $\epsilon = \advfrac/6$ in Lemma~\ref{lem:22concentration} gives Condition~\ref{con2}.

Finally, we show the last condition.
To show it we use $\epsilon = \advfrac$ in Lemma~\ref{lem:lemsqdev}.
From Condition~\ref{con}, note that for $b\in \bgood\setminus \bgood^{d}(\alphsubset,\advfrac)$
\[
|\bempprob-\med| \le |\bempprob-\subsetprobtarget|+|\subsetprobtarget-\med| \le 2 \sqrt{\frac{\ln (6e/\advfrac)}\bsize}+ \sqrt{\frac{\ln 6}\bsize}\le  3\sqrt{\frac{\ln (6e/\advfrac)}\bsize},
\]
Then, for $b\in \bgood\setminus \bgood^{d}(\alphsubset,\advfrac)$, from the definition of corruption score it follows that $\indcorruption =0$. Next set of inequalities complete the proof of condition~\ref{con3}.
\begin{align*}
    \corruption(\bgood) &= \sum_{b\in \bgood}\indcorruption = \sum_{b\in \bgood\setminus \bgood^{d}(\alphsubset,\advfrac)}\indcorruption+ \sum_{b\in  \bgood^{d}(\alphsubset,\advfrac)}\indcorruption \\
    &= \sum_{b\in  \bgood^{d}(\alphsubset,\advfrac)}\indcorruption \\
    &\overset{\text{(a)}}\le \sum_{b\in  \bgood^{d}(\alphsubset,\advfrac)}(\bempprob- \med)^2\\
    &= \sum_{b\in  \bgood^{d}(\alphsubset,\advfrac)}(\bempprob- \subsetprobtarget + \subsetprobtarget -\med)^2\\
    &\overset{\text{(b)}}\le \sum_{b\in  \bgood^{d}(\alphsubset,\advfrac)}(\bempprob- \subsetprobtarget)^2+ \sum_{b\in  \bgood^{d}(\alphsubset,\advfrac)}(\med- \subsetprobtarget)^2\\
    &+ 2\sqrt{\Bigg(\sum_{b\in  \bgood^{d}(\alphsubset,\advfrac)}(\bempprob- \subsetprobtarget)^2\Bigg)\Bigg(\sum_{b\in  \bgood^{d}(\alphsubset,\advfrac)}(\med- \subsetprobtarget)^2\Bigg)}\\
    &\overset{\text{(c)}}\le \frac{\advfrac}{2}\bgoodsize {\frac{ \ln (e/\advfrac) } {\bsize}}+\frac{\advfrac}{40}\bgoodsize {\frac{ \ln 6 } {\bsize}}+\sqrt{\frac{\advfrac}{2}\bgoodsize {\frac{ \ln (e/\advfrac) } {\bsize}}\times\frac{\advfrac}{40}\bgoodsize {\frac{ \ln 6 } {\bsize}}}< {\advfrac}\bgoodsize {\frac{ \ln (e/\advfrac) } {\bsize}},
\end{align*}
here (a) follows from the definition of the corruption score, (b) uses Cauchy-Schwarz inequality  and (c) follows from Lemma~\ref{lem:lemsqdev} and Condition~\ref{con}.

\section{Proof of the other Lemmas}\label{App:proofs}
We first prove an auxiliary Lemma that will be useful in other proofs. For a given sub-collection $\Bsc$ and subset $\alphsubset$, the next lemma bounds the total squared distance of $\bempprob$ from $\subsetprobtarget$ for adversarial batches in $\Bsc$ in terms of corruption score $\genoverallcorruptionsub$.
\begin{lemma}\label{lem:sqdevadv}
Suppose the conditions~\ref{con} and~\ref{con3} holds. For subset $\alphsubset$, let $\genoverallcorruptionsub =t\cdot\kappa_G$, for some $t\ge 0$,
then
\[
(t-3-2\sqrt{t})\kappa_G\le \sum_{b\in  \Bsc\cap\badv}(\bempprob- \subsetprobtarget)^2 \le (t+17+2\sqrt{t})\kappa_G.
\]
\end{lemma}
\begin{proof}
For the purpose of this proof, let $\Gsc = \Bsc\cap\bgood$ and $\Asc = \Bsc\cap\badv$. Then
\begin{align}\label{eq:local1}
&\sum_{b\in \Asc}(\bempprob- \subsetprobtarget)^2 =  \sum_{b\in \Asc: \indcorruption > 0}(\bempprob- \subsetprobtarget)^2+\sum_{b\in \Asc: \indcorruption = 0}(\bempprob- \subsetprobtarget)^2
\end{align}

From the definition of corruption score, for batch $b\in \Bsc$, with zero corruption score $\indcorruption$, we have $|\bempprob- \med|\le 3 \sqrt{\frac{\ln (6e/\advfrac)}\bsize}$. Then using Condition~\ref{con} and the triangle inequality, for such batches with zero corruption score, we get
\begin{equation}\label{eq:local2}
    |\bempprob- \subsetprobtarget|\le \sqrt{\ln (6)/\bsize}+ 3 \sqrt{\frac{\ln (6e/\advfrac)}\bsize}< 4 \sqrt{\frac{\ln (6e/\advfrac)}\bsize}.
\end{equation}

Next,
\begin{align}
&\sum_{b\in \Asc: \indcorruption > 0}(\bempprob- \subsetprobtarget)^2 \nonumber\\
&=\sum_{b\in \Asc: \indcorruption > 0}(\bempprob-\med+\med -\subsetprobtarget)^2 \nonumber\\
&\overset{\text{(a)}}\le \sum_{b\in \Asc: \indcorruption > 0}(\bempprob- \med)^2 + \sum_{b\in \Asc: \indcorruption > 0}(\med- \subsetprobtarget)^2\nonumber\\
&+2\sqrt{\Bigg(\sum_{b\in \Asc: \indcorruption > 0}(\bempprob- \med)^2\Bigg)\Bigg( \sum_{b\in \Asc: \indcorruption > 0}(\med- \subsetprobtarget)^2\Bigg)}\nonumber\\
&\overset{\text{(b)}}\le  \sum_{b\in \Asc}\indcorruption + \sum_{b\in \Asc}\frac{\ln 6}{\bsize}\nonumber+2\sqrt{\Bigg(\sum_{b\in \Asc} \indcorruption \Bigg)\Bigg( \sum_{b\in \Asc}\frac{\ln 6}{\bsize}\Bigg)}\nonumber\\
&\overset{\text{(c)}}\le  \corruption(\Asc,\alphsubset) + \kappa_G+2\sqrt{\corruption(\Asc,\alphsubset)\cdot\kappa_G},\label{eq:local3}
\end{align}
here (a) uses Cauchy-Schwarz inequality, (b) follows from the definition of corruption score and Condition~\ref{con}, and (c) uses $|\Asc|\le \advfrac\btotal$ and $(\advfrac\btotal\ln 6)/n\le \kappa_G$. 

A similar calculation as the above leads to the following
\begin{align}
&\sum_{b\in \Asc: \indcorruption > 0}(\bempprob- \subsetprobtarget)^2 
\ge  \corruption(\Asc,\alphsubset) -2\sqrt{\corruption(\Asc,\alphsubset)\cdot\kappa_G},\label{eq:local4}
\end{align}

Next, we show the upper bound in the lemma. Combining  equations~\eqref{eq:local1},~\eqref{eq:local2} and~\eqref{eq:local3} gives
\begin{align*}
\sum_{b\in \Asc}(\bempprob- \subsetprobtarget)^2 &\le \corruption(\Asc,\alphsubset) + \kappa_G+2\sqrt{\corruption(\Asc,\alphsubset)\cdot\kappa_G}+\sum_{b\in \Asc: \indcorruption = 0} 4 \sqrt{{\frac{\ln (6e/\advfrac)}\bsize}}\\
&\le \genoverallcorruptionsub + \kappa_G+2\sqrt{\genoverallcorruptionsub\cdot\kappa_G}+16 |\badv| {\frac{\ln (6e/\advfrac)}\bsize}\\
&\le (t+17+2\sqrt{t})\kappa_G, 
\end{align*}
here the second last inequality used $\Asc\subseteq\Bsc$ and $\Asc\subseteq\badv$. This completes the proof of the upper bound.

To prove the lower bound, we first note that 
\begin{align*}
    &\genoverallcorruptionsub = \sum_{b\in \Bsc}\indcorruption= \sum_{b\in \Gsc}\indcorruption+\sum_{b\in \Asc}\indcorruption\\
     &\le \sum_{b\in \bgood}\indcorruption+ \corruption(\Asc,\alphsubset)\\
     &\le \corruption(\bgood)+ \corruption(\Asc,\alphsubset)\le {\frac{ \advfrac \btotal\ln ({6e}/\advfrac) } {\bsize}}+\sum_{b\in \Asc}\indcorruption,
\end{align*}
here the last inequality uses condition~\ref{con3}.
The above equation implies that
\begin{align}\label{eq:intsas}
     \corruption(\Asc,\alphsubset) \ge \genoverallcorruptionsub - \advfrac\btotal {\frac{\ln (6e/\advfrac)}\bsize}= (t-1)\kappa_G.
\end{align}
By combining, equations~\eqref{eq:local1}~\eqref{eq:local4} and~\eqref{eq:intsas}, we get the lower bound
\begin{align*}
\sum_{b\in \Asc}(\bempprob- \subsetprobtarget)^2 &\ge (t-1)\kappa_G -2\sqrt{|t-1|\kappa_G\cdot\kappa_G} = (t-1-2\sqrt{|t-1|})\kappa_G \ge(t-3-2\sqrt{t})\kappa_G .
\end{align*}
\end{proof}

\subsection{Proof of Lemma~\ref{lem:corruptionandell1}}

\begin{proof}
For the purpose of this proof, let $\Gsc = \Bsc\cap\bgood$ and $\Asc = \Bsc\cap\badv$. Note that $|\Bsc| \ge |\Gsc|\ge (1-\advfrac/6)\bgood$.

Fix subset $\alphsubsetdef$. Next,
\begin{align*}
& \Uempprob-\subsetprobtarget = \frac{1}{|\Bsc|}\sum_{b\in \Bsc}\bempprob- \subsetprobtarget= \frac{1}{|\Bsc|}\sum_{b\in \Bsc}(\bempprob- \subsetprobtarget)\\
& = \frac{1}{|\Bsc|}\sum_{b\in \Gsc}(\bempprob- \subsetprobtarget) + \frac{1}{|\Bsc|}\sum_{b\in \Asc}(\bempprob- \subsetprobtarget) \\
& = \frac{|\Gsc|}{|\Bsc|} (\UGempprob-\subsetprobtarget)+ \frac{1}{|\Bsc|}\sum_{b\in \Asc}(\bempprob- \subsetprobtarget) 
\end{align*}
Therefore,
\begin{align}
 |\Uempprob-\subsetprobtarget|&\le  \frac{|\Gsc|}{|\Bsc|}  |\UGempprob-\subsetprobtarget|+  \frac{1}{|\Bsc|} \sum_{b\in \Asc}|\bempprob- \subsetprobtarget|\nonumber\\
&\overset{\text{(a)}}\le \frac \advfrac 2\sqrt{\frac{\ln (6e/\advfrac) } {\bsize}}+\frac{1}{|\Bsc|}\sum_{b\in \Asc}|\bempprob- \subsetprobtarget|\nonumber\\
  &\overset{\text{(b)}}\le \frac \advfrac 2\sqrt{\frac{\ln (6e/\advfrac) } {\bsize}}+\frac{1}{|\Bsc|} \sqrt{|\Asc|\sum_{b\in \Asc}(\bempprob- \subsetprobtarget)^2} \nonumber\\
  &\overset{\text{(c)}}\le \frac \advfrac 2\sqrt{\frac{\ln (6e/\advfrac) } {\bsize}}+\frac{1}{|\Bsc|} \sqrt{|\Asc|\cdot(t+17+2\sqrt{t})\kappa_G}\nonumber\\
   &\overset{\text{(d)}}\le \frac \advfrac 2\sqrt{\frac{\ln (6e/\advfrac) } {\bsize}}+\frac{1}{|\Bsc|} \sqrt{|\Asc|\cdot(t+17+2\sqrt{t})\frac{ \advfrac \btotal\ln ({6e}/\advfrac) } {\bsize}}\nonumber\\
   &\le \frac \advfrac 2\sqrt{\frac{\ln (6e/\advfrac) } {\bsize}}+ \sqrt{\frac{|\Asc|\cdot
   \btotal}{|\Bsc|^2}\cdot(t+17+2\sqrt{t})\frac{ \advfrac \ln ({6e}/\advfrac) } {\bsize}},\label{eq:usefulloc}
\end{align}
here in (a) uses Condition~\ref{con2} and $|\Gsc|\le |\Bsc|$, inequality (b) follows from the Cauchy-Schwarz inequality, inequality (c) uses Lemma~\ref{lem:sqdevadv}, and (d) uses the definition of $\kappa_G$.
Let $|\Asc| = |\badv|- \text{D}$, for some $\text{D}\in [0, |\badv|]$. Also from Lemma note that 
\[
|\Gsc| \ge (1-\advfrac/6)\bgoodsize = \bgoodsize - \bgoodsize\advfrac/6 = \bgoodsize - \btotal\advfrac(1-\advfrac)/6.\]
Therefore,
\begin{align*}
\frac{|\Asc|\cdot\btotal}{|\Bsc|^2}= \frac{|\Asc|\cdot\btotal}{(|\Asc|+|\Gsc|)^2} &\le \frac{(|\badv|- \text{D})\btotal}{(|\badv|- \text{D}+\bgoodsize - \btotal\advfrac(1-\advfrac)/6)^2} \\
&= \frac{(\advfrac\btotal- \text{D})\btotal}{(\btotal- \text{D} - \btotal\advfrac(1-\advfrac)/6)^2} \\
&\overset{\text{(a)}}\le \frac{(\advfrac\btotal- \text{D})\btotal}{(\btotal- \text{D} - 0.04 \btotal)^2}\\
&\overset{\text{(b)}}\le \frac{\advfrac\btotal^2}{(0.96\btotal)^2}\le \frac\advfrac{0.96^2}, 
\end{align*}
here (a) follows since $\advfrac(1-\advfrac)$ takes maximum value at $\advfrac = 0.4$ in range $\advfrac\in (0,0.4]$, and (b) follows since the expression is maximized at $\text D = 0$.

Then combining above equation with~\eqref{eq:usefulloc} gives
\begin{align}
|\Uempprob-\subsetprobtarget|&
\le \frac \advfrac 2\sqrt{\frac{\ln (6e/\advfrac) } {\bsize}}+ \sqrt{(t+17+2\sqrt{t})\frac{ \advfrac^2 \ln ({6e}/\advfrac) } {0.96^2\bsize}}\nonumber\\
&\le   \Big(1/2+\frac{1}{0.96}\sqrt{(t+17+2\sqrt{t})}\Big)\advfrac\sqrt{\frac{\ln (6e/\advfrac) } {\bsize}}\label{eq:impimp}\\
&\overset{\text{(a)}}\le   \Big(5+\sqrt{2.1 t}\Big)\advfrac\sqrt{\frac{\ln (6e/\advfrac) } {\bsize}},\label{eq:usus}
\end{align}
here inequality (a) uses the fact that $2t^{1/2}\le t+1$ and $\sqrt{x^2+y^2}\le|x|+|y|$.
Finally, using the definition of $L_1$ distance between two distributions complete the proof of the Theorem.
\end{proof}

\subsection{Proof of Lemma~\ref{lem:bounddel}}
\begin{proof}
From the second statement in Lemma~\ref{lem:alg2char}, each batch that gets removed is adversarial with probability $\ge 0.95$. Batch deletion deletes more than $0.1\advfrac\btotal$ good batches in total over all runs iff it samples $0.1\advfrac\btotal$ good batches in first $0.1\advfrac\btotal+|\badv|$ batches removed as otherwise all adversarial batches would have been exhausted already and Batch deletion algorithm would not remove batches any further. 
But the expected number of good batches sampled is 
$\le 0.05(\times0.1\advfrac\btotal+|\badv|)\le 0.005\advfrac\btotal+0.05\advfrac\btotal< 0.06\advfrac\btotal$. 

Then using the Chernoff-bound, probability of sampling (removing) more than $0.1\advfrac\btotal$ good batches in $0.1\advfrac\btotal+|\badv|$ deletions is $\le e^{-O(\advfrac\btotal)}\le e^{-O(\alphsize)}$.
Hence, with high probability the algorithm deletes less than $0.1\advfrac m =  0.6 \advfrac m/6 \le \bgoodsize\advfrac/6 $ batches.
\end{proof}

\subsection{Proof of Lemma~\ref{lem:trans}}
\begin{proof}
For the purpose of this proof, let $\Gsc = \Bsc\cap\bgood$ and $\Asc = \Bsc\cap\badv$. For batches $b$ in a sub-collection $\Bsc$, the next equation relates the empirical variance of $\bempprob$ to sum of their squared deviation from $\subsetprobtarget$.
\begin{align}
& |\Bsc|\empvarsub {\Bsc} = \sum_{b\in \Bsc}(\bempprob- \Uempprob)^2
= \sum_{b\in \Bsc}(\bempprob-\subsetprobtarget- (\Uempprob-\subsetprobtarget) )^2\nonumber\\
&= \sum_{b\in \Bsc}\Big((\bempprob- \subsetprobtarget )^2 +(\Uempprob-\subsetprobtarget)^2-2(\Uempprob-\subsetprobtarget)(\bempprob- \subsetprobtarget )\Big)\nonumber\\
&= \sum_{b\in \Bsc}(\bempprob- \subsetprobtarget )^2 +|\Bsc|(\Uempprob-\subsetprobtarget)^2-2(\Uempprob-\subsetprobtarget)\sum_{b\in \Bsc} (\bempprob- \subsetprobtarget )\nonumber\\
&= \sum_{b\in \Bsc}(\bempprob- \subsetprobtarget )^2 +|\Bsc|(\Uempprob-\subsetprobtarget)^2-2(\Uempprob-\subsetprobtarget)(|\Bsc|\Uempprob- |\Bsc|\subsetprobtarget )\nonumber\\
&= \sum_{b\in \Bsc}(\bempprob- \subsetprobtarget )^2 -|\Bsc|(\Uempprob-\subsetprobtarget)^2\nonumber\\
&= \sum_{b\in \Asc}(\bempprob- \subsetprobtarget )^2+\sum_{b\in \Gsc}(\bempprob- \subsetprobtarget)^2-|\Bsc|(\subsetprobtarget-\Uempprob)^2.\label{eq:totalvarianceeq}
\end{align}
The next set of inequalities lead to the upper bound in the Lemma.
\begin{align*}
&|\Bsc|(\empvarsub \Bsc- \var{\Uempprob})\\
&\overset{\text{(a)}}= \sum_{b\in \Asc}(\bempprob- \subsetprobtarget )^2+\sum_{b\in \Gsc}(\bempprob- \subsetprobtarget)^2-|\Bsc|(\subsetprobtarget-\Uempprob)^2- |\Bsc|\var{\Uempprob}\\
&\overset{\text{(b)}}\le (t+17+2\sqrt{t})\kappa_G+ |\Gsc|\var{\subsetprobtarget}+|\Gsc|\frac{ 6\advfrac\ln (\frac{6e}\advfrac) } {\bsize} - |\Bsc|\var{\Uempprob}\\
&\overset{\text{(c)}}\le  (t+17+2\sqrt{t})\kappa_G+6 \advfrac\btotal {\frac{\ln (6e/\advfrac)}\bsize}+ |\Bsc|\var{\subsetprobtarget} - |\Bsc|\var{\Uempprob}\\
&\overset{\text{(d)}}\le   (t+23+2\sqrt{t})\kappa_G+ \btotal\frac{|\subsetprobtarget-\Uempprob|}{\bsize},
\end{align*}
here inequality (a) follows from~\eqref{eq:totalvarianceeq}, (b) follows from Lemma~\ref{lem:sqdevadv} and condition~\ref{con2}, and (c) follows since $|\Gsc|\le|\Bsc|$ and $\var \cdot \ge 0$, and inequality (d) uses \eqref{eq:fineq} and $|\Bsc|\le \btotal$.
Next, from equation~\eqref{eq:usus} we have, 
\begin{align}
&|\Uempprob-\subsetprobtarget| \le (5+\sqrt{2.1 t}) \advfrac\sqrt{{\frac{\ln (6e/\advfrac)}\bsize} }\nonumber\\
&=(5+\sqrt{2.1 t})\advfrac \ln (6 e/\advfrac) \sqrt{ \frac{1}{\bsize\ln (6 e/\advfrac)}} \nonumber\\
&\le (5+\sqrt{2.1 t})\frac{\bsize \kappa_G}{\btotal}.\label{eq:hbsj}
\end{align}
Combining the above two equations gives the upper bound in the Lemma.

Next showing the lower bound,
\begin{align*}
&|\Bsc|(\empvarsub {\Bsc}- \var{\Uempprob})\\
&\overset{\text{(a)}}= \sum_{b\in \Asc}(\bempprob- \subsetprobtarget )^2+\sum_{b\in \Gsc}(\bempprob- \subsetprobtarget)^2-|\Bsc|(\subsetprobtarget-\Uempprob)^2- |\Bsc|\var{\Uempprob}\\
&\overset{\text{(b)}}\ge (t-3-2\sqrt{t})\kappa_G+|\Gsc|\var{\subsetprobtarget}-|\Gsc|\frac{ 6\advfrac\ln (\frac{6e}\advfrac) } {\bsize}-|\Bsc|(\subsetprobtarget-\Uempprob)^2- |\Bsc|\var{\Uempprob}\\
&\ge (t-9-2\sqrt{t})\kappa_G+|\Gsc|\var{\subsetprobtarget}-|\Bsc|(\subsetprobtarget-\Uempprob)^2- |\Gsc|\var{\Uempprob}-|\Asc|\var{\Uempprob}\\
&\ge (t-9-2\sqrt{t})\kappa_G-|\Gsc|(\var{\Uempprob}-\var{\subsetprobtarget})-|\Bsc|(\subsetprobtarget-\Uempprob)^2-|\Asc|\var{\Uempprob}\\
& \overset{\text{(c)}}\ge  (t-9-2\sqrt{t})\kappa_G- |\Gsc|\frac{|\subsetprobtarget-\Uempprob|}{\bsize} -\frac{|\Asc|}{4\bsize}-|\Bsc|(\subsetprobtarget-\Uempprob)^2\\
& \ge  (t-9-2\sqrt{t})\kappa_G- \btotal\frac{|\subsetprobtarget-\Uempprob|}{\bsize} -\frac{\advfrac\btotal}{4\bsize}-\btotal(\subsetprobtarget-\Uempprob)^2\\
& \overset{\text{(d)}}\ge  (t-15-2\sqrt{t}-\sqrt{2.1 t})\kappa_G-\btotal(\subsetprobtarget-\Uempprob)^2,
\end{align*}
here inequality (a) follows from~\eqref{eq:totalvarianceeq}, (b) follows from Lemma~\ref{lem:sqdevadv} and condition~\ref{con2}, (c) follows from \eqref{eq:fineq} and $\var \cdot \le \frac1{4\bsize}$, and inequality (d) follows from \eqref{eq:hbsj}.

Next, we bound the last tem in the above equation to complete the proof. From equation~\eqref{eq:impimp},
\begin{align*}
(\subsetprobtarget-\Uempprob)^2 &\le \Big(1/2+\frac{1}{0.96}\sqrt{(t+17+2\sqrt{t})}\Big)^2\advfrac^2{\frac{\ln (6e/\advfrac) } {\bsize}} \\
&\le \Big(1/4+\frac{1}{0.96^2}{(t+17+2\sqrt{t})}+\frac{1}{0.96}\sqrt{(t+17+2\sqrt{t})} \Big)\advfrac\cdot\kappa_G\\
&\overset{\text{(a)}}\le \Big(1/4+1.1{(t+17+2\sqrt{t})}+5+\sqrt{2.1 t} \Big)\advfrac\cdot\kappa_G\\
&\le \Big(24+1.1t+2.2\sqrt{t}+\sqrt{2.1 t}\Big)\advfrac\cdot\kappa_G\\
&\overset{\text{(b)}}\le 0.4\Big(24+1.1t+2.2\sqrt{t}+\sqrt{2.1 t}\Big)\kappa_G,
\end{align*}
here inequality (a) uses the fact that $2t^{1/2}\le t+1$ and $\sqrt{x^2+y^2}\le|x|+|y|$ and inequality (b) uses $\advfrac\le 0.4$. Combining above two equations give us the lower bound in the Lemma.
\end{proof}

\section{Proof of Theorem~\ref{th:algper}}
First, we restate the statement of the main theorem.

\begin{theorem}\label{th:algperapp}
Suppose the conditions~\ref{con}-~\ref{con3} holds. Then Algorithm~\ref{alg1} runs in polynomial time and with probability $\ge 1-O(e^{-k})$ return a sub-collection $\Bsc_{f} \subseteq\allbatches$ such that $|\Bsc_f\cap\bgood| \ge (1 -\frac\advfrac6) \bgoodsize$ and for $\targetdis^* = {\bar{\targetdis}_{\Bsc_f}}$,
\[
||\targetdis^* - \targetdis ||_1 \le  100 \advfrac \sqrt{\frac{\ln (6e/\advfrac) } {\bsize}}.
\]
\end{theorem}
\begin{proof}
Lemma~\ref{lem:bounddel} show that for the sub-collection $\Bsc_i$ at each iteration $i$, $|\Bsc_i\cap\bgood| \ge (1 -\frac\advfrac6) \bgoodsize$, hence, for sub-collection $\Bsc_f$ returned by the algorithm $|\Bsc_f\cap\bgood| \ge (1 -\frac\advfrac6) \bgoodsize$, with probability $\ge 1-O(e^{-k})$. This also implies that the total number of deleted batches are $< (1+1/6)\advfrac\btotal$. 

To complete the proof of the above Theorem, we state the following corollary, which is a direct consequence of Lemma~\ref{lem:trans}. 
\begin{corollary}\label{cor:numeric}
Suppose the conditions~\ref{con}-~\ref{con3} holds. Then following hold for any $\Bsc\subseteq \allbatches$ such that $|\Bsc\cap\bgood| \ge (1 -\frac\advfrac6) \bgoodsize$.
\begin{enumerate}
    \item $|\empvarsub {\Bsc}- \var{\Uempprob}| \ge 75\kappa_G$ implies that $\genoverallcorruptionsub \ge 25\kappa_G$.
    \item $|\empvarsub {\Bsc}- \var{\Uempprob}| \le 150\kappa_G$ implies that $\genoverallcorruptionsub\le 900\kappa_G$. 
\end{enumerate}
\end{corollary}

In each iteration of Algorithm~\ref{alg1}, except the last, ${Detection-Algorithm}$ returns a subset for which the difference between two variance estimate is $\ge  75\kappa_G$. The first statement in the above corollary implies that corruption is high for this subset. Batch Deletion removes batches from the sub-collection to reduce the corruption for such subset. From Statement 3 of Lemma~\ref{lem:alg2char}, in each iteration Batch Deletion removes $\ge 25\kappa_G-20\kappa_G$ batches. Since the total batches removed are $<7/6\advfrac\btotal$, this implies that the algorithm runs for at-max $\frac{7\advfrac\btotal}{6\times 5\kappa_G}< n$ iterations. 

The algorithm terminates when ${Detection-Algorithm}$ returns a subset for which the difference between two variance estimate is $\le 75\kappa_G$. Then Lemma~\ref{lem:polalg} implies that the difference between two variance estimate is $\le 150 \kappa_G$ for all subsets. Then the above corollary shows that corruption for all subsets is $\le 900\kappa_G$. Therefore, $\genoverallcorruption\le 900\kappa_G$. Then Lemma~\ref{lem:corruptionandell1} bounds the $L_1$ distance.
\end{proof}

\end{document}